\documentclass[preprint]{article}

\usepackage[accepted]{aistats2026}
\usepackage[round]{natbib}

\usepackage{multirow}
\usepackage{changepage}
\usepackage{xcolor} 
\usepackage{xspace}
\usepackage{mdframed}
\usepackage[table]{xcolor}      % For \cellcolor and coloring table cells
\usepackage{subcaption}         % For \subfloat or sub-tables
\usepackage{caption}            % Better control of captions
\usepackage{array}              % For more flexible table columns
\usepackage{lmodern}            % Optional: for better fonts
\usepackage{pifont}
\usepackage{array}
\usepackage{amssymb}
\usepackage{mathtools}
\usepackage{dsfont}
\usepackage{quoting} 
\usepackage{fancyhdr}
\usepackage{url}
\usepackage{amsmath}
\usepackage{amsthm}
\usepackage{hhline}
\usepackage{graphicx}
\usepackage{enumitem}
\usepackage{algorithm}
\usepackage{algorithmic}
\usepackage{booktabs}
\usepackage{nameref}

\definecolor{myblue}{RGB}{0,160,255}

\newcommand{\argmax}{\operatorname{argmax}}

\newcommand{\mathornot}[1]{\ifmmode #1 \else $#1$\fi}

\renewenvironment{proof}
{
\par\noindent
\textbf{Proof. }\ignorespaces
}
{
\hfill$\square$\par
}

\newcommand{\bis}{\operatorname{bi}}
\newcommand{\row}{\operatorname{row}}
\newcommand{\col}{\operatorname{col}}
\newcommand{\all}{\operatorname{all}}

\newcommand{\myargmin}[1]{\quad\underset{\mathclap{\footnotesize\begin{array}{c}#1\end{array}}}{\operatorname{argmin}}\quad\ }

\newcommand{\appropto}{\mathrel{\vcenter{
  \offinterlineskip\halign{\hfil$##$\cr
    \propto\cr\noalign{\kern2pt}\sim\cr\noalign{\kern-2pt}}}}}

\definecolor{color1}{RGB}{31,119,180}
\definecolor{color2}{RGB}{255,127,14}
\definecolor{color3}{RGB}{44,160,44}
\definecolor{color4}{RGB}{214,39,40}
\definecolor{color5}{RGB}{148,103,189}
\definecolor{color6}{RGB}{140,86,75}
\definecolor{color7}{RGB}{227,119,194}
\definecolor{color8}{RGB}{127,127,127}
\definecolor{color9}{RGB}{188,189,34}
\definecolor{color10}{RGB}{23,190,207}
\definecolor{gold}{RGB}{245, 185, 0}
\definecolor{silver}{RGB}{5, 145, 233}
\definecolor{bronze}{RGB}{218, 57, 80}

\definecolor{Dark2A}{rgb}{0.105, 0.620, 0.467} % #1b9e77
\definecolor{Dark2B}{rgb}{0.851, 0.373, 0.008} % #d95f02
\definecolor{Dark2C}{rgb}{0.459, 0.439, 0.702} % #7570b3
\definecolor{Dark2D}{rgb}{0.906, 0.161, 0.541} % #e7298a
\definecolor{Dark2E}{rgb}{0.400, 0.651, 0.118} % #66a61e
\definecolor{Dark2F}{rgb}{0.902, 0.671, 0.008} % #e6ab02
\definecolor{Dark2G}{rgb}{0.651, 0.463, 0.114} % #a6761d
\definecolor{Dark2H}{rgb}{0.400, 0.400, 0.400} % #666666

\usepackage{afterpage}

\newtheorem{proposition}{Proposition}
\newtheorem{definition}{Definition}
\newtheorem{lemma}{Lemma}

\begin{document}

%\runningtitle{On the Normalization of Confusion Matrices: Methods and Geometric Interpretations}

\runningauthor{Johan Erbani, Sonia Ben Mokhtar, Pierre-Edouard Portier, Elöd Egyed-Zsigmond, Diana Nurbakova}

\twocolumn[

\aistatstitle{On the Normalization of Confusion Matrices: Methods and Geometric Interpretations}

\aistatsauthor{ Johan Erbani \And Sonia Ben Mokhtar \And  Pierre-Edouard Portier}
\aistatsaddress{LIRIS, INSA Lyon,\\CNRS (UMR 5205), France \And LIRIS, INSA Lyon,\\CNRS (UMR 5205), France \And Caisse d'Epargne Rhône Alpes,\\Tour Incity, Lyon, France}

\aistatsauthor{Elöd Egyed-Zsigmond  \And  Diana Nurbakova}

\aistatsaddress{LIRIS, INSA Lyon,\\CNRS (UMR 5205), France \And LIRIS, INSA Lyon,\\CNRS (UMR 5205), France} ]

\begin{abstract}
The confusion matrix is a standard tool for evaluating classifiers, providing a detailed view of model errors. In heterogeneous settings, its entries are influenced by two main factors: class similarity, reflecting how easily the model confuses certain classes, and distribution bias, stemming from imbalanced training or test distributions. Because confusion matrix values jointly reflect both factors, it is difficult to disentangle their individual effects. To address this issue, we introduce bi-normalization via Iterative Proportional Fitting, a generalization of row and column normalization. Unlike standard approaches, this method recovers the underlying structure of class similarity. By disentangling error sources, it enables a more precise diagnosis of model behavior and facilitates classifier improvement. We further establish connections between normalization, importance sampling, and class representations in the model’s latent space, thus offering a clearer interpretation of normalization schemes. Our implementation is publicly available\footnote{\url{https://github.com/JohanErbani/Bi-Normalization}}.
\end{abstract}

\section{INTRODUCTION}
The confusion matrix is a key tool for understanding, evaluating, and improving a model’s behavior~\citep{krstinic2020multi, krstinic2024multi, gortler2022neo, erbani2024confusion}. It is a $C \times C$ square matrix that provides a detailed view of model errors, where $C$ is the number of classes. Each entry $(i,j)$ counts the number of instances with true label $i$ that the model predicted as $j$, for all $i,j \in \{1, \dots, C\}$. It is typically built from the test set~\citep{sammut2011encyclopedia}: starting from a zero matrix, the entry $(i,j)$ is incremented by $1$ whenever a sample with label $i$ is predicted as $j$.

Two main factors shape confusion matrix values: (i) class similarity and (ii) distribution bias~\citep{erbani2024confusion}, as illustrated in Figure~\ref{EFFECTS}. (i) Class similarity refers to confusion between similar classes—the more alike two classes $i$ and $j$ are, the more frequently they are misclassified as each other, increasing values in entries $(i,j)$ and/or $(j,i)$ (Subfigure~\ref{class effect}). While this source of error is intrinsic to the classes themselves, it also reflects inherent model properties: regardless of training, these similarities persist. (ii) Distribution bias refers to errors caused by class imbalance in the predictions or test set. Models trained on imbalanced data tend to overpredict majority classes and underpredict minority ones~\citep{buda2018systematic, leevy2018survey}, inflating the columns of majority classes (Subfigure~\ref{prediction effect}). Similarly, imbalance in the test set affects the rows: majority classes have higher counts, while minority classes have lower ones (Subfigure~\ref{testset effect}). This source of confusion is accidental and varies with training.

Understanding the sources of classification errors enables targeted improvements. For example, errors from class similarity can be reduced by enhancing preprocessing to better distinguish similar classes or by augmenting the training set with more examples of those classes~\citep{ghosh2024class, tian2024novel, temraz2022solving, aggarwal2021minority}. Overprediction errors can be reduced by adjusting the loss function, such as reweighting class contributions to reduce the influence of frequent classes~\citep{fernando2021dynamically, deepak2023brain, chamseddine2022handling, wu2022deep}. However, confusion matrix values reflect a mix of both class similarity and distribution bias (Subfigure~\ref{all}), making it difficult to disentangle their individual contributions.

\begin{figure}
\centering
\caption{How class similarity and distribution bias affect confusion matrices. The first three subfigures isolate each factor, while the last combines them: (a) Test set imbalance, with B as the majority class and C as the minority; (b) Prediction imbalance, with C over-predicted and B under-predicted; (c) Class similarities; (d) Combined influence of all factors. Darker colors indicate higher values.}
\label{EFFECTS}
\begin{subfigure}[b]{0.45\linewidth}
\centering
\caption{Test set: Imbalance, Predictions: Balance, Class similarity: No}\label{testset effect}
\scriptsize
\begin{tabular}{|l|l l l |} \hline 
 & A & B & C \\ \hline
A & \cellcolor{myblue!18} & \cellcolor{myblue!18} & \cellcolor{myblue!18} \\
B & \cellcolor{myblue!33} & \cellcolor{myblue!33} & \cellcolor{myblue!33} \\
C & \cellcolor{myblue!3} & \cellcolor{myblue!3} & \cellcolor{myblue!3} \\ \hline
\end{tabular}
\end{subfigure}
\hfill
\begin{subfigure}[b]{0.45\linewidth}
\centering
\caption{Test set: Balance, Predictions: Imbalance, Class similarity: No}\label{prediction effect}
\scriptsize
\begin{tabular}{|l|l l l |} \hline 
 & A & B & C \\ \hline
A & \cellcolor{myblue!18} & \cellcolor{myblue!3} & \cellcolor{myblue!33} \\ 
B & \cellcolor{myblue!18} & \cellcolor{myblue!3} & \cellcolor{myblue!33} \\
C & \cellcolor{myblue!18} & \cellcolor{myblue!3} & \cellcolor{myblue!33} \\
\hline
\end{tabular}
\end{subfigure}\\

\vspace{.5cm}\begin{subfigure}[b]{0.45\linewidth}
\centering
\caption{Test set: Balance, Predictions: Balance, Class similarity: Yes}\label{class effect}
\scriptsize
\begin{tabular}{|l|l l l |} \hline 
 & A & B & C \\ \hline
A & \cellcolor{myblue!48} & \cellcolor{myblue!3} & \cellcolor{myblue!3} \\
B & \cellcolor{myblue!3} & \cellcolor{myblue!33} & \cellcolor{myblue!18} \\
C & \cellcolor{myblue!3} & \cellcolor{myblue!18} & \cellcolor{myblue!33} \\ \hline
\end{tabular}
\end{subfigure}
\hfill%
\begin{subfigure}[b]{0.45\linewidth}
\centering
\caption{Test set: Imbalance, Predictions: Imbalance, Class similarity: Yes}\label{all}
\scriptsize
\begin{tabular}{|l|l l l |} \hline 
 & A & B & C \\ \hline
A & \cellcolor{myblue!84} & \cellcolor{myblue!24} & \cellcolor{myblue!54} \\
B & \cellcolor{myblue!54} & \cellcolor{myblue!69} & \cellcolor{myblue!84} \\
C & \cellcolor{myblue!24} & \cellcolor{myblue!24} & \cellcolor{myblue!69} \\ \hline
\end{tabular}
\end{subfigure}
\end{figure}

In practice, machine learning workflows often apply row, column, or all normalization to confusion matrices:
\begin{equation*}
\row(M)_{ij}\! =\! \frac{M_{ij}}{M_{i+}},\ \col(M)_{ij}\! =\!\frac{M_{ij}}{M_{+j}},\ \all(M)_{ij}\! =\!\frac{M_{ij}}{M_{++}},
\end{equation*}
where $M$ is a confusion matrix, $M_{i+}=\sum_jM_{ij}$, $M_{+j}=\sum_i M_{ij}$, and $M_{++}=\sum_{i,j}M_{ij}$. Many papers present several normalized matrices, as each normalization highlights different aspects of the model behavior~\citep{gortler2022neo}. However, in imbalanced settings, these normalizations may fail to capture class similarities, as shown in our experiments.

Class similarities in the confusion matrix become visible when all rows and columns have the same total, which typically occurs when both the test and training sets are balanced. When both the test set and predictions are imbalanced, a double normalization of rows and columns at the same time is needed to reveal these similarities. This normalization can be achieved using the Iterative Proportional Fitting (IPF) procedure—also known as the Sinkhorn-Knopp algorithm, biproportional fitting, the RAS method, or simply matrix scaling~\citep{idel2016review}. IPF is widely used in statistics, economics, and computer science~\citep{idel2016review}. We refer to the normalization produced by IPF as \textit{bi-normalization}, denoted by $\bis$. This normalization yields a confusion matrix in which each row and each column sums to $1$: $\bis(M)_{i+} = \bis(M)_{+j} = 1$.

Standard normalizations have a probabilistic interpretation: $\row$, $\col$, and $\all$ correspond to $\mathbb{P}(\hat{Y} = j \mid Y = i)$, $\mathbb{P}(Y = i \mid \hat{Y} = j)$, and $\mathbb{P}(Y = i, \hat{Y} = j)$, respectively~\citep{gortler2022neo}, where $Y$ and $\hat{Y}$ denote the random variables of labels and predictions. We introduce two complementary perspectives that provide further insight into matrix normalization.

\textbf{Importance sampling.} Normalization can be viewed as applying an importance sampling strategy: it reweights each label–prediction pair to match a desired distribution—for instance, balancing labels via row normalization.

\textbf{Model representation.} Normalization is also related to how the model encodes class representations. We show empirically that the overlap of class clusters in the latent space corresponds—depending on how it is measured—to a particular form of normalized confusion matrix.

Standard and bi normalizations can thus be interpreted through these perspectives, providing additional probabilistic and geometric meaning.

Our main contributions are:
\begin{itemize}[noitemsep, topsep=0pt]
\item We show that bi-normalization generalizes row and column normalizations while satisfying key expected properties.
\item We establish connections between confusion matrix normalizations—$\row$, $\col$, $\all$, and $\bis$—and both importance sampling and class representations in the model’s latent space, offering clearer interpretations of normalized confusion matrices.
\item We provide empirical evidence that bi-normalization reveals class similarities more effectively than other normalization methods and supports our geometric interpretation.
\end{itemize}

\section{RELATED WORK}
The first part of this state-of-the-art reviews the application of IPF procedure to normalize contingency tables, while the second focuses on the impact of distribution bias on the confusion matrix. 

\subsection{Normalization Using Iterative Proportional Fitting}
\citet{deming1940least} proposed the use of the Iterative Proportional Fitting (IPF) procedure to estimate cell probabilities in a contingency table under given marginal constraints. 

\citet{ireland1968contingency} address the problem of estimating a new contingency table $q$ from a given contingency table $p$, where each cell $q_{ij}$ represents a probability, subject to known and fixed marginal probabilities $\sum_j q_{ij}=u_i$ and $\sum_i q_{ij}=v_j$. The authors introduce an algorithm minimizing the KL divergence from $p$ to $q$. This procedure is equivalent to IPF~\citep{idel2016review}.

In the fields of remote sensing and geographic modeling, disagreement between maps and reality is commonly displayed in a confusion matrix~\citep{congalton2001accuracy, hardin1997statistical}. When multiple classification or modeling methods are used, the resulting confusion matrices are typically compared to assess significant differences~\citep{hardin1997statistical}. 
Confusion matrix normalization using IPF is a standard analytical technique~\citep{congalton1991review}. 
In this way, differences in sample sizes used to generate the matrices are eliminated and, therefore, individual cell values within the matrix are directly comparable~\citep{congalton1991review}. 

For speaker verification system improvements, \citet{nagineni2010line} use the IPF procedure to normalize confusion matrices and select a cohort set based on similarity modeling for each client speaker.

\subsection{Causes of Errors}
In the context of neural networks, \citet{erbani2024confusion} examine how imbalanced training and test sets influence confusion matrix entries. They introduce the test–training ranking to distinguish sources of error. When some entries deviate from this criterion, it suggests strong similarity between the corresponding classes. However, class similarities can still impact errors even when the ranking is preserved. Moreover, deviations from the ranking are ambiguous—they do not indicate which entry is too high or too low~\citep{erbani2024confusion}.

\subsection{Key Takeaways}
Normalizing contingency tables—particularly confusion matrices—using IPF procedure is not a new idea; it is even considered a standard approach in some disciplines. However, to the best of our knowledge, no existing work in machine learning employs this normalization to separate the effects of distribution bias and class similarity within confusion matrix. This decomposition provides deeper insights into the sources of errors, helping to assess and improve the classifier.

To the best of our knowledge, no prior work has connected confusion matrix normalization to importance sampling or to the structure of class representations in a model’s latent space. We introduce two perspectives—probabilistic and geometric—that offer deeper insight into classifier behavior.

\section{BI-NORMALIZATION}
In this section, we define the bi-normalization, outline the key properties it should satisfy, and show how the IPF procedure can compute it accordingly.

Theoretical analysis is simpler with positive matrices. Since a non-negative confusion matrix $M$ and its strictly positive counterpart $M + \epsilon$ (for small $\epsilon$) exhibit similar model behavior, we propose using $M + \epsilon$ instead. We now assume that $M$ is a positive confusion matrix.

\subsection{Empirical Definition \& Desirable Properties}\label{Informal Definition and Desirable Properties}
Normalization approximates the confusion matrix that would arise under specific experimental conditions. Row normalization corresponds to the confusion matrix obtained from a balanced test set, whereas column normalization corresponds to that obtained from a model producing balanced predictions.

\textbf{Empirical definition.} Bi-normalization captures class similarities by approximating the confusion matrix under balanced label distributions in both the training and test sets. We formalize this by requiring both row and column normalization: $\bis(M)_{i+} = \bis(M)_{+j} = 1$.

Standard normalization methods $\all$, $\row$, and $\col$ satisfy three properties—idempotence, class distribution invariance, and information preservation—described below. By extension, bi-normalization is expected to satisfy these properties as well.

\textbf{Idempotence.} Normalization maps the matrix to specific experimental conditions. Once these conditions are reached, further normalization does not modify the matrix. Accordingly, bi-normalization is expected to satisfy $\bis \circ \bis(M) = \bis(M)$, where $\circ$ denotes composition.

\textbf{Class Distribution Invariance.} If two confusion matrices differ only in the experimental conditions targeted by the normalization operator, then normalization yields similar matrices.

For instance, row normalization targets a balanced test set. Thus, regardless of the test set label distribution, a trained model yields similar row-normalized confusion matrices. Formally, let $S$ and $T$ be the confusion matrices obtained from $\Phi$ on two large test sets with different label distributions (without class extinction).
Then, the error patterns across rows remain similar, i.e., $S_{i,:} \appropto T_{i,:}$ for all $i$, where $\appropto$ denotes equality up to a positive scaling factor\footnote{For vectors or matrices $U$ and $V$ of the same size, $U \appropto V$ means $U \approx a V$ for some $a \in \mathbb{R}_{>0}$.}. This implies that $\row(S) \approx \row(T)$, showing that the row operator is invariant to variations in the test set label distribution.

This invariance is formalized as invariance under left multiplication by a positive diagonal matrix: $\row(A M) = \row(M)$ for any diagonal matrix $A \in \mathbb{R}_{>0}^{C \times C}$. Similarly, column normalization is invariant under right multiplication by a positive diagonal matrix: $\col(M B) = \col(M)$ for any diagonal matrix $B \in \mathbb{R}_{>0}^{C \times C}$.

Bi-normalization should be invariant under both operations, as it combines row and column normalization, i.e., $\bis(A M B) = \bis(M)$.

\textbf{Information Preservation.} Normalization infers, from the original confusion matrix, the one obtained under other experimental conditions. It should therefore preserve as much of the original information as possible, altering it only when necessary.

This requirement can be formalized as an optimization problem that promotes similarity between the original and normalized matrices. We use the Kullback–Leibler (KL) divergence as a dissimilarity measure. Given row and/or column constraints, the normalized matrix is defined as the one closest to the original in terms of KL divergence.

For instance, row normalization satisfies
\begin{equation*}
\row(M)\in
\myargmin{
P \in \mathbb{R}_{>0}^{C \times C}: \\
P_{i+} =1,\ P_{+j} = \sum_{i}\frac{M_{ij}}{M_{i+}} \ \forall\, i,j}
D_{\mathrm{KL}}(P \| M),
\end{equation*}
where $M \in \mathbb{R}_{>0}^{C \times C}$ is the original confusion matrix, and $D_{\mathrm{KL}}(P \| M)$ denotes the KL divergence from $P$ to $M$. Proofs and extensions to other normalization methods are provided in the Appendix. 

Bi-normalization should minimize this optimization problem under the constraints that each row and each column sums to one.

\subsection{Theoretical Definition, Properties \& Estimation}\label{Formal Definition}
This subsection provides the formal definition of the bi-normalization, its properties and how to estimate it (see Appendix for details). 

\begin{definition}
$\bis(M)$ is the unique minimizer of the following constrained problem\footnote{This optimization problem always has a solution, and it is unique, as shown in the appendix.}:
\begin{equation*}
\bis(M)\in 
\myargmin{
P \in \mathbb{R}_{>0}^{C \times C}: \\
P_{i+} = P_{+j} = 1 \ \forall\, i,j
}
D_{\mathrm{KL}}(P \| M)
\end{equation*}
\end{definition}

We now establish key properties:
\begin{proposition}\label{pp}
$\bis(M)$ satisfies idempotence, class distribution invariance, and information preservation as described in Subsection~\nameref{Informal Definition and Desirable Properties}.
\end{proposition}
The last property is satisfied by definition, while the two others follow directly from it, as shown in the appendix.

Although $\operatorname{bi}(M)$ has no closed form, it can be estimated using the IPF algorithm (Algorithm in Appendix). In our case, it consists of normalizing the matrix alternately by rows and columns until convergence\footnote{Let $Q^{(t)}$ be the matrix produced by IPF at iteration $t$. Convergence is reached when $Q^{(t)}_{i+}$ and $Q^{(t)}_{+j}$ are sufficiently close to $1$ under a fixed tolerance.}. For example, if convergence is reached in two iterations, then
$
\col \circ \row \Big( \col \circ \row(M)\Big) \approx \bis(M).
$

The IPF procedure always converges when the input matrix $M$ is positive~\citep{idel2016review}. Let $M_{i+} = u_i$ and $M_{+j} = v_j$ for all $i$ and $j$, and denote by $Q$ the limit of the IPF procedure\footnote{Let $Q^{(t)}$ be the matrix produced by IPF at iteration $t$. Then $Q^{(t)} \to Q$ as $t \to \infty$.}. Then, $Q$ solves the following optimization problem~\citep{kurras2015symmetric, idel2016review}:
\begin{equation*}
Q \in
\myargmin{
P \in \mathbb{R}_{>0}^{C \times C}: \\
P_{i+} = u_i, P_{+j} = v_j \ \forall\, i,j
}  
D_{\mathrm{KL}}(P \| M)
\end{equation*}

This directly implies the following:
\begin{proposition}
$\operatorname{bi}(M)$ can be approximated using the IPF procedure with row and column constraints set to $1$. We have $\operatorname{bi}(M) = Q \approx \hat{Q}$, where $Q$ is the theoretical IPF limit and $\hat{Q}$ is its empirical estimate.
\end{proposition}
Since $\operatorname{bi}(M)$ is doubly stochastic\footnote{A doubly stochastic matrix is a square matrix with nonnegative entries such that each row and each column sums to $1$.} and $M$ is positive, the IPF procedure converges linearly~\citep{idel2016review}.

\section{PROBABILISTIC \& GEOMETRIC VIEWS OF NORMALIZATION}
This section introduces two perspectives on normalization: (i) importance sampling and (ii) class representations. We first show that normalization can be interpreted as importance sampling. We then propose to model class representations as volumes, and finally link the overlaps of these volumes to normalized confusion matrices.

\subsection{Normalization as Importance Sampling}\label{IS}
The normalization process can be viewed as an importance sampling strategy to match the desired distribution.

Let $(x_1, y_1), \ldots, (x_N, y_N)$ be the dataset used to construct the confusion matrix $M$, where $(x_k, y_k)\in \mathcal{X} \times [C]$, $\mathcal{X}$ is the input space, $C$ is the number of classes, and $[C]=\{1,\ldots,C\}$ is the set of class indices. For each $k=1,\ldots,N$, the model prediction $\hat{y}_k \in [C]$ is defined by
$$
\hat{y}_k \in \argmax_{i=1}^C \Phi(x_k)_i,
$$
where $\Phi$ denotes the model. 

\textbf{Weighted sum.} By definition, the confusion matrix is
$
M = \sum_{k=1}^N E_{y_k \hat{y}_k},
$
where $E_{ij}$ denotes the $C \times C$ matrix with a $1$ at entry $(i,j)$ and $0$ elsewhere. Each matrix $E_{y_k\hat{y}_k}$ is a confusion matrix constructed from the single pair $(y_k,\hat{y}_k)$, capturing its individual contribution to $M$.

Standard and bi normalizations can be obtained by multiplying $M$ on the left and right by diagonal matrices. Specifically, let $D^l$ and $D^r$ denote diagonal matrices (with $l$ for left and $r$ for right), and let $M$ be the original confusion matrix. Then, $D^l M D^r$ yields the normalized version of $M$. 

For instance, setting
$$
D^l=\operatorname{diag}(\frac{1}{M_{1+}}, \ldots, \frac{1}{M_{C+}})\text{ and }D^r=I
$$
yields the row-normalized matrix, where $\operatorname{diag}(d_1,\ldots,d_C)$ denotes a diagonal matrix with entries $d_1,\ldots,d_C$ on its diagonal, and $I$ the identity. Other normalizations are presented in the Appendix.

As a result, normalization can be expressed as a weighted sum of the individual contributions $E_{y_k\hat{y}_k}$:
$$
D^l M D^r = \sum_{k=1}^N D^l_{y_k}D^r_{\hat{y}_k}E_{y_k \hat{y}_k},
$$
where each contribution $E_{y_k \hat{y}_k}$ is weighted by the product $D^l_{y_k} D^r_{\hat{y}_k}$. 

For clarity, the weight $D^l_{y_k} D^r_{\hat{y}_k}$ can be expressed through a weight function $\omega$. For instance, the weight function $\omega_r$ (with $r$ for row) for row-normalization is 
$$
\omega_r : (y,\hat{y}) \mapsto \frac{1}{M_{y+}},\ \row(M) = \sum_{k=1}^N \omega_r(y_k,\hat{y}_k) E_{y_k \hat{y}_k}.
$$
In the same way, we define $\omega_{a}$ ($a$ for all), $\omega_c$ ($c$ for column) and $\omega_b$ ($b$ for bi), yielding all, column, and bi normalizations, respectively (see Appendix for explicit definitions). 

When the dataset is viewed as a realization of random variables, this process corresponds to importance sampling.

\textbf{Importance Sampling.} Let $(Y_1, \hat{Y}_1), \ldots, (Y_N, \hat{Y}_N)$ be independent and identically distributed (i.i.d.) random pairs, representing label–prediction pairs, and $f$ be their joint distribution $f(i,j) = \mathbb{P}(Y=i,\hat{Y}=j)$.
The random counterpart of $\all(M)$ converges almost surely (a.s.) to
$$
\frac{1}{N}\sum_{k=1}^N E_{Y_k\hat{Y}_k}\underset{N\rightarrow \infty}{\overset{\text{a.s.}}{\longrightarrow}}\mathbb{E}[E_{Y\hat{Y}}],
$$
where
$$
\mathbb{E}[E_{Y\hat{Y}}]_{ij} = \mathbb{E}[\mathds{1}_{Y=i, \hat{Y}=j}]=\mathbb{P}(Y=i,\hat{Y}=j)=f(i,j),
$$
with $\mathds{1}$ the indicator function.

Importance sampling refers to Monte Carlo methods that approximate expectations under a target distribution $g$ by reweighting samples drawn from a proposal distribution $f$~\citep{tokdar2010importance}. This approach is, for instance, useful when direct samples from $g$ are unavailable. To approximate the expectation of $E_{Y,\hat{Y}}$ when $(Y,\hat{Y})$ is distributed according to $g$ rather than $f$, we use the estimator $\hat{\mu}_g$, defined as
$$
\begin{aligned}
\hat{\mu}_g=&\tfrac{1}{N}\sum_{k=1}^N \tfrac{g(Y_k,\hat{Y}_k)}{f(Y_k,\hat{Y}_k)} E_{Y_k\hat{Y}_k} \\
&\underset{N\rightarrow \infty}{\overset{\text{a.s.}}{\longrightarrow}}\mathbb{E}[\tfrac{g(Y,\hat{Y})}{f(Y,\hat{Y})}E_{Y\hat{Y}}]
=\mathbb{E}_{g}[E_{Y\hat{Y}}]
\end{aligned}
$$
with
$$
\begin{aligned}
\mathbb{E}_{g}[E_{Y\hat{Y}}]_{ij} \!=\!\mathbb{E}_g[\mathds{1}_{Y=i, \hat{Y}=j}]
\!=\!\mathbb{P}_{g}(Y=i,\hat{Y}=j)\!=\!g(i,j),
\end{aligned}
$$
where $\mathbb{E}_{g}$ and $\mathbb{P}_{g}$ denote expectation and probability assuming the variables follow the law $g$. 

For importance sampling to remain accurate, the variance of each entry of the estimator $\hat{\mu}_g$ must remain small, which occurs when $f$ is approximately proportional to $g$~\citep{tokdar2010importance}.

By selecting different choices of $g$, standard normalizations arise naturally. As an example, consider row normalization (see Appendix for other normalizations). Setting $g(i,j) = \mathbb{P}(\hat{Y}=j \mid Y=i)$\footnote{Note that $g$ is not a normalized distribution, since $\sum_{i,j} g(i,j) = C$ rather than $1$.} recovers the row-normalized confusion matrix.

In this case, the importance sampling ratio becomes
$$
\frac{g(i,j)}{f(i,j)} = \frac{\mathbb{P}(\hat{Y}=j\mid Y=i)}{\mathbb{P}(Y=i,\hat{Y}=j)} = \frac{1}{\mathbb{P}(Y=i)}\approx \frac{N}{M_{i+}}.
$$
Accordingly, the empirical estimator is
$$
\begin{aligned}
\tfrac{1}{N}\sum_{k=1}^N \tfrac{N}{M_{y_k+}} E_{y_k \hat{y}_k} = \sum_{k=1}^N \omega_r(y_k,\hat{y}_k) E_{y_k \hat{y}_k}
=\row(M).
\end{aligned}
$$

For row normalization to be reliable, the importance sampling ratio should remain controlled. In particular, this requires that each class $i$ is sufficiently represented in the test set; otherwise, the estimator exhibits high variance and becomes unstable.

\textbf{Key Takeaways.} Normalization is an importance sampling procedure: each pair is reweighted to shift the empirical distribution toward a target one. Reliability requires that the two distributions are not too dissimilar.

\subsection{Class Representations as Volumes}\label{Class Representations as Volumes}
This subsection introduces a geometric perspective on class representations in the model’s latent space. The key idea is to characterize class clusters as multidimensional histograms.

\textbf{Representation Space \& Dimensionality Reduction.} 
We use the output of the final pre-logit layer as the class representation space, a procedure already employed in previous work~\citep{beery2020synthetic,krizhevsky2012imagenet,zhu2022deep}. 
Let $\varepsilon$ be the embedding function, so that $\Phi(x) = \operatorname{Softmax} \circ \operatorname{Logit} \circ\ \varepsilon(x)$, where $\varepsilon(x) \in \mathbb{R}^n$ for each $(x, y)$ in the training set $\mathcal{D}$.
 
To facilitate histogram construction, we project the embeddings into a lower-dimensional space using Principal Component Analysis (PCA). Let $P$ be the projection matrix with $m < n$, so that $P\varepsilon(x) \in \mathbb{R}^m$. 

\textbf{Class Clusters \& Histograms.} For each label $i$, we define the cluster of projected embedded points as 
$
\mathcal{C}_i = \{P\varepsilon(x) : (x, y) \in \mathcal{D},\ y = i\}.
$
Similarly, the cluster of points predicted as class $j$ is 
$
\hat{\mathcal{C}}_j = \{P\varepsilon(x) : (x, y) \in \mathcal{D},\ \hat{y} = j\}.
$

For each label and prediction clusters, we construct a multidimensional, non-overlapping histogram~\citep{thaper2002dynamic} (see Appendix for details). Each histogram is interpreted as a volume in $\mathbb{R}^{m+1}$, providing a geometric view of the clusters. Figure~\ref{fig:voxels} shows a toy example illustrating histograms of class clusters. 

\begin{figure}[ht]
    \centering
    \includegraphics[width=.7\linewidth]{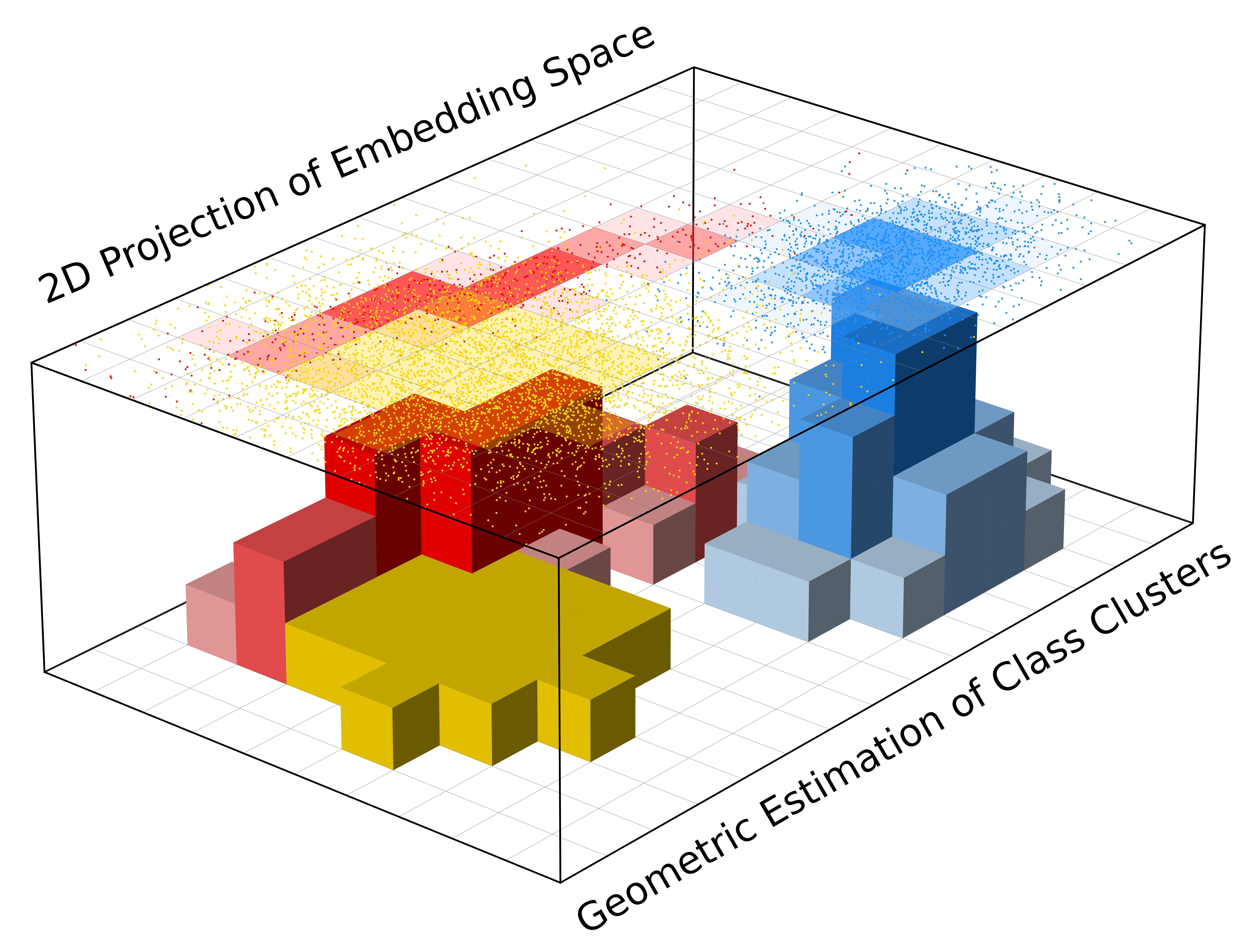}
    \caption{Toy example showing a 2D projection of embedded observations forming class clusters (top map). Each cluster is associated with a histogram indicating the spatial distribution of points (bottom map).}
    \label{fig:voxels}
\end{figure}

\textbf{Weighted Histograms.} Following the importance sampling view (Subsection~\ref{IS}), we weight each point to construct weighted histograms, using the schemes $\omega_a$, $\omega_r$, $\omega_c$, and $\omega_b$. For a given bin, its height is defined not by the number of points it contains, but by the sum of their weights. 

For instance, consider cluster $\mathcal{C}_i$ and the weighting $\omega_c$. The height of a given bin $b$ is
$
\sum_{(x,y)\in\mathcal{D}} \omega_c(y,\hat{y})\mathds{1}_{y=i} \mathds{1}_{P\varepsilon(x)\in b},
$
whereas in the unweighted histogram it is
$
\sum_{(x,y)\in\mathcal{D}} \mathds{1}_{y=i} \mathds{1}_{P\varepsilon(x)\in b}.
$

We denote by $V(\mathcal{C}, \omega)$ the hypervolume in $\mathbb{R}^{m+1}$ of the weighted histogram built from cluster $\mathcal{C}$ using the weighting $\omega$.

\subsection{Geometric View of Normalized Confusion Matrices}\label{geom}
This subsection introduces a new matrix, referred to as Geometric Confusion Matrix (GCM), which quantifies the overlap between the weighted histogram of label and prediction clusters:
\begin{equation*}
\begin{aligned}
\bigg(\operatorname{GCM}_{\omega}\bigg)_{ij} &= \lambda\bigg(V(\mathcal{C}_i, \omega)\cap V(\hat{\mathcal{C}}_j, \omega)\bigg),\\
\end{aligned}
\end{equation*}
for $i,j=1,\dots, C$, where $\lambda$ denotes the Lebesgue measure. The Lebesgue measure of $V(\mathcal{C}_i, \omega)\cap V(\hat{\mathcal{C}}_j, \omega)$ corresponds to the volume of the intersection between the $\omega$-weighted histograms of $\mathcal{C}_i$ and the one of $\hat{\mathcal{C}}_j$. Unlike the standard confusion matrix, the GCM takes into account the spatial organization of embedded points (see Appendix for details).

We now provide geometric views of confusion matrix normalizations, all claims are supported by experimental results (see Section~\ref{Empirical Results}).

\textbf{All-normalization.} All-normalized matrix approximates the overlap between the unweighted histograms of label and prediction clusters:
$\all(M) \appropto \operatorname{GCM}_{\omega_a}$,
where $\appropto$ means approximately equal up to a scaling factor\footnote{Let $M$ and $N$ be matrices of the same size. We write $M \appropto N$ to indicate that $M \approx \alpha N$ for some scalar $\alpha \in \mathbb{R}_{>0}$.}.

\textbf{Row-normalization.} Row-normalized matrix approximates the overlap between normalized label histograms and the resulting weighted prediction histograms:
$\row(M) \appropto \operatorname{GCM}_{\omega_r}$.

More precisely, up to a known scaling factor, the label histograms are normalized in the sense that $\lambda\big(V(\mathcal{C}_i, \omega_r)\big) = 1$, whereas the volume of the prediction histograms, $\lambda\big(V(\hat{\mathcal{C}}_j, \omega_r)\big)$ are not fixed a priori (see Appendix for details). Accordingly, label histograms are normalized, while prediction histograms are scaled according to $\omega_r$.

\textbf{Col-normalization.} Similarly, column-normalized matrix approximates the overlap between normalized prediction histograms and resulting weighted label histograms:
$\col(M) \appropto \operatorname{GCM}_{\omega_c}$.

\textbf{Bi-normalization.} Finally, bi-normalized matrix approximates the overlap between normalized histograms of both label and prediction clusters:
$\bis(M) \appropto \operatorname{GCM}_{\omega_b}$.

These correspondences provide a geometric interpretation of normalized confusion matrices and illustrate how confusion matrices capture the organization of classes in the model’s latent space.

\section{EXPERIMENTAL SETUP}\label{Experimental Setup}
This section describes our experimental setup. 

\subsection{Datasets \&  Heterogeneity} 
We conduct experiments on four datasets: MNIST~\citep{lecun1998gradient}, Fashion-MNIST~\citep{xiao2017fashion}, CIFAR-10~\citep{krizhevsky2009learning}, and STL-10~\citep{STL10}. 

Normalizations methods differ mainly on non-diagonal matrices. For example, row, column, and bi normalizations yield the same result on diagonal matrices. To encourage non-diagonal confusion matrices, we make MNIST and Fashion-MNIST harder by applying random rotations of 0°, 90°, 180°, or 270° during preprocessing, applied to both the training and test images.

Bi-normalization is particularly useful under heterogeneous settings. To simulate data heterogeneity, we sample from the original datasets with a Dirichlet distribution of concentration $\alpha$, following~\citep{allouah2023fixing, hsu2019measuring, erbani2025weighted}. We consider five levels—very low, low, medium, high, and extreme—corresponding to $\alpha \in \{10, 3, 1, 0.3, 0.1\}$.

To avoid class extinction, we enforce a minimum representation of $20\%$ of the original class size. For example, in a dataset with $1000$ samples per class, at least $200$ samples are retained for each class, while the remaining samples are allocated according to a Dirichlet distribution.

\subsection{Models \& Training}
CNN models (see Appendix for details) are trained using stochastic gradient descent with cross-entropy loss, a batch size of $32$, a learning rate of $10^{-3}$, a momentum of $0.9$, and a weight decay of $10^{-4}$.

As mentioned above, the main differences between normalizations appear in the off-diagonal confusion matrices. Therefore, training is limited to a maximum of $10$ epochs, and stops earlier if the classifier reaches $60\%$ balanced test accuracy.

\subsection{Metric}\label{Metric}
Let $S$ and $T$ be two confusion matrices in $\mathbb{R}_{\geq 0}^{C \times C}$, possibly already normalized.

To measure the similarity between confusion matrices, we use their overlap:
$$
\operatorname{Overlap}(S, T) = \sum_{ij} \min\left(\all(S)_{ij}, \all(T)_{ij}\right),
$$
where $\min$ denotes the element-wise minimum (see Appendix for the rationale behind reusing $\all$ normalization). This score ranges from $0$ to $1$, with higher values indicating greater similarity. It reaches $1$ if and only if $T = S$.

This metric is strictly equivalent to the $L^1$ distance, since $\|\all(S) - \all(T)\|_1 = 2 - 2\operatorname{Overlap}(S, T)$ (see Appendix for proof).

\subsection{Baselines \& Experiments}\label{Baselines and Experiments}
We compare bi-normalization with standard normalizations: $\row$, $\col$, and $\all$. Each experiment is repeated over MNIST, Fashion-MNIST, CIFAR-10, and STL-10, with five heterogeneity levels and $30$ random seeds.

\textbf{Experiment 1.} This experiment compares the confusion matrix obtained from balanced datasets (denoted as $M_1$) with normalized versions of the confusion matrix obtained from imbalanced datasets (denoted as $M_2$).
\begin{enumerate}[noitemsep, topsep=0pt]
\item Initialize a model with a fixed random seed, then create two deep copies: model 1 and model 2.
\item Train model 1 on a balanced training set and compute its confusion matrix $M_1$ using a balanced test set.
\item Sample an imbalanced training set and test set, train model 2 on the imbalanced training set, and compute its confusion matrix $M_2$ on the imbalanced test set.
\item Apply normalization techniques to $M_2$ and evaluate their similarity to $M_1$.
\end{enumerate}
This experiment tests whether bi-normalization reveals class similarities under imbalanced data more effectively than other approaches.

\textbf{Experiment 2.} This experiment, conducted in an imbalanced setting, compares different versions of the GCM—$\operatorname{GCM}\!{\omega_a}, \operatorname{GCM}\!{\omega_r}, \operatorname{GCM}\!{\omega_c}, \operatorname{GCM}\!{\omega_b}$—with normalized confusion matrices derived from $M$—$\row(M)$, $\col(M)$, $\all(M)$, and $\bis(M)$.
\begin{enumerate}[noitemsep, topsep=0pt]
\item Initialize the model with a fixed seed.
\item Sample imbalanced training and test sets. Train the model. Compute both the GCM variants and the normalized confusion matrices using the test set.
\item Measure pairwise similarities.
\end{enumerate}
This experiment evaluates how each normalization reveals the structure of class representations in the model’s latent space under a specific weighted histogram.

Details on the construction of multivariate histograms required for GCM are provided in the Appendix.

\begin{figure}[!ht]
    \centering
    % First subfigure
    \begin{subfigure}[b]{0.99\linewidth}
        \centering
        \includegraphics[width=0.99\linewidth]{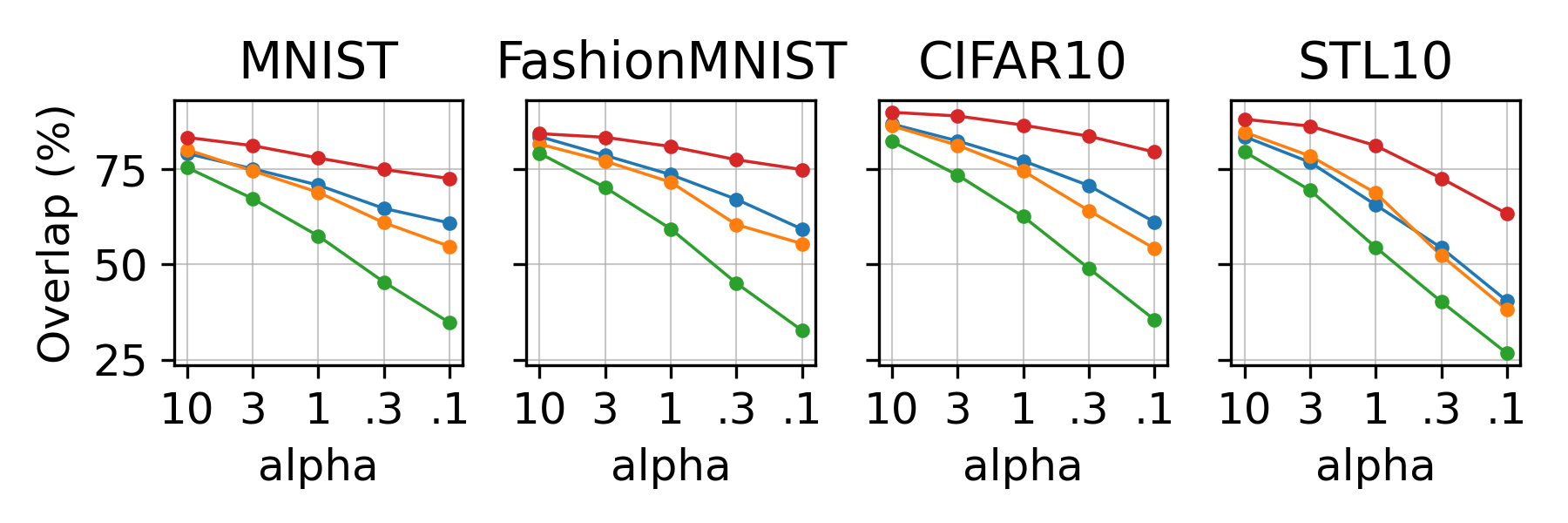}
        \vspace{-.6cm}
        \caption{Target: Matrix from balanced setting}
        \label{fig:subfig1}
    \end{subfigure}\\
    \vspace{.3cm}
    % Second subfigure
    \begin{subfigure}[b]{0.99\linewidth}
        \centering
        \includegraphics[width=0.99\linewidth]{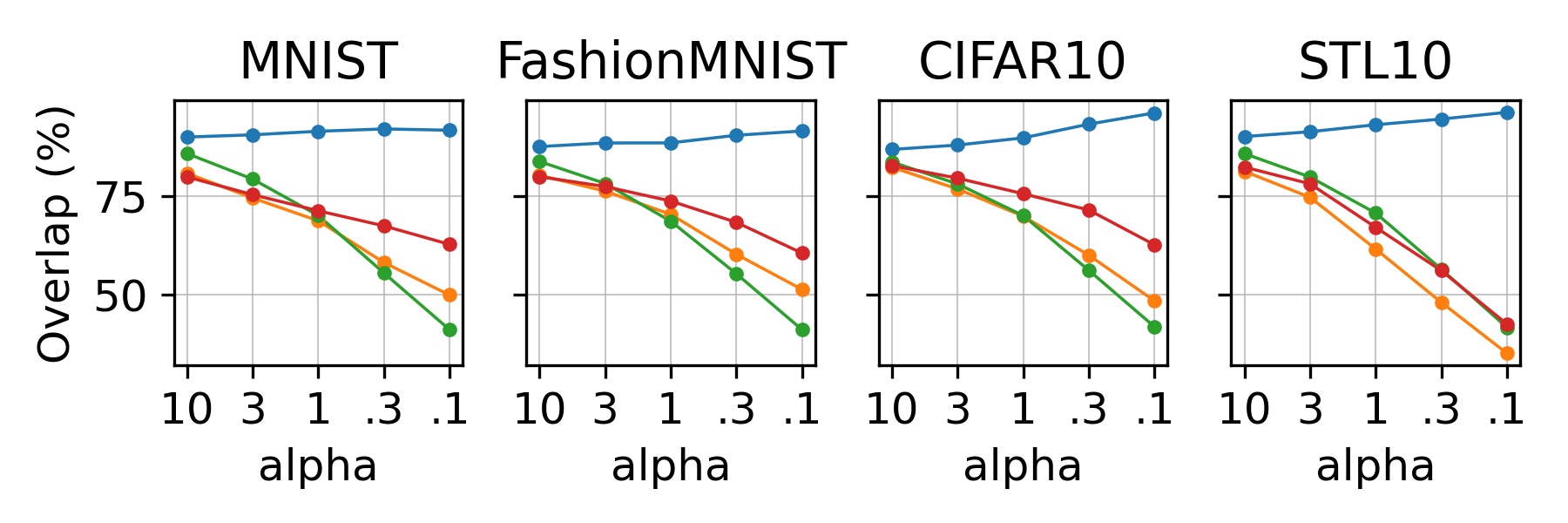}
        \vspace{-.6cm}
        \caption{Target: $\operatorname{GCM}\!{\omega_r}$}
        \label{fig:subfig2}
    \end{subfigure}\\
    \vspace{.3cm}
    % Third subfigure
    \begin{subfigure}[b]{0.99\linewidth}
        \centering
        \includegraphics[width=0.99\linewidth]{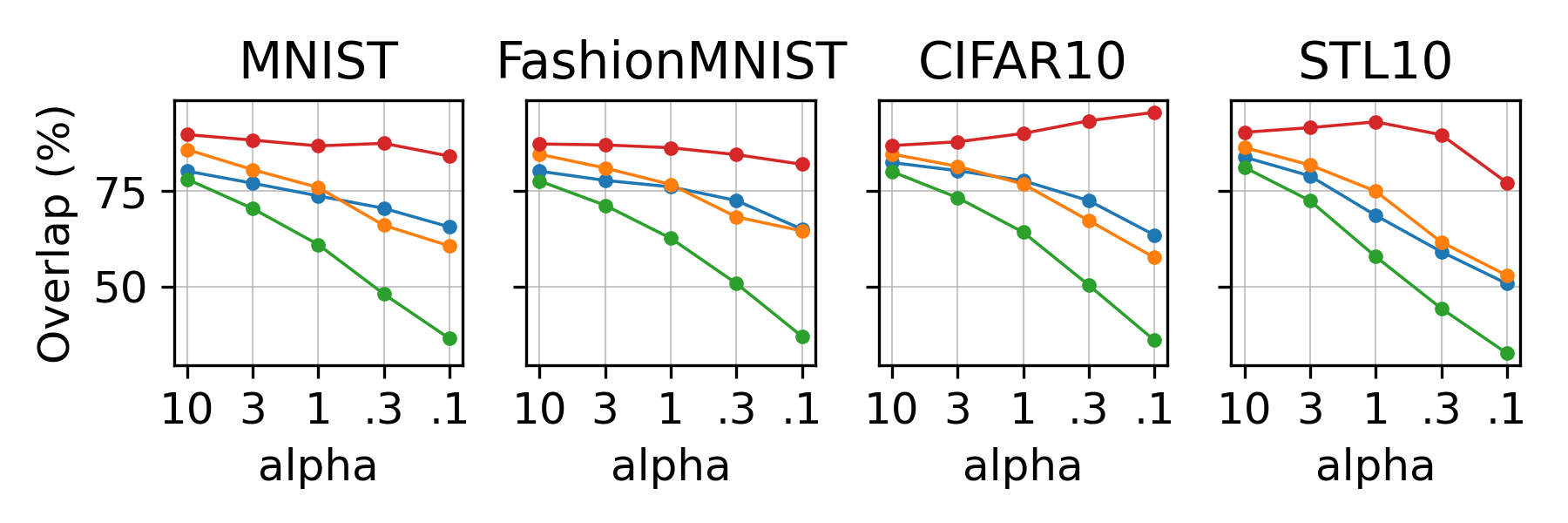}
        \vspace{-.6cm}
        \caption{Target: $\operatorname{GCM}\!{\omega_b}$}
        \label{fig:subfig3}
    \end{subfigure}
    \caption{Overlap (\%) between target matrices and normalized matrices. Legend:~\textcolor{color4}{$\;\bullet\!$}~bi
    \textcolor{color1}{$\;\bullet\!$}~row 
    \textcolor{color2}{$\;\bullet\!$}~col 
    \textcolor{color3}{$\;\bullet\!$}~all
    }
    \label{plot}
\end{figure}

\begin{figure}[!ht]
    \centering
    \includegraphics[width=0.8\linewidth]{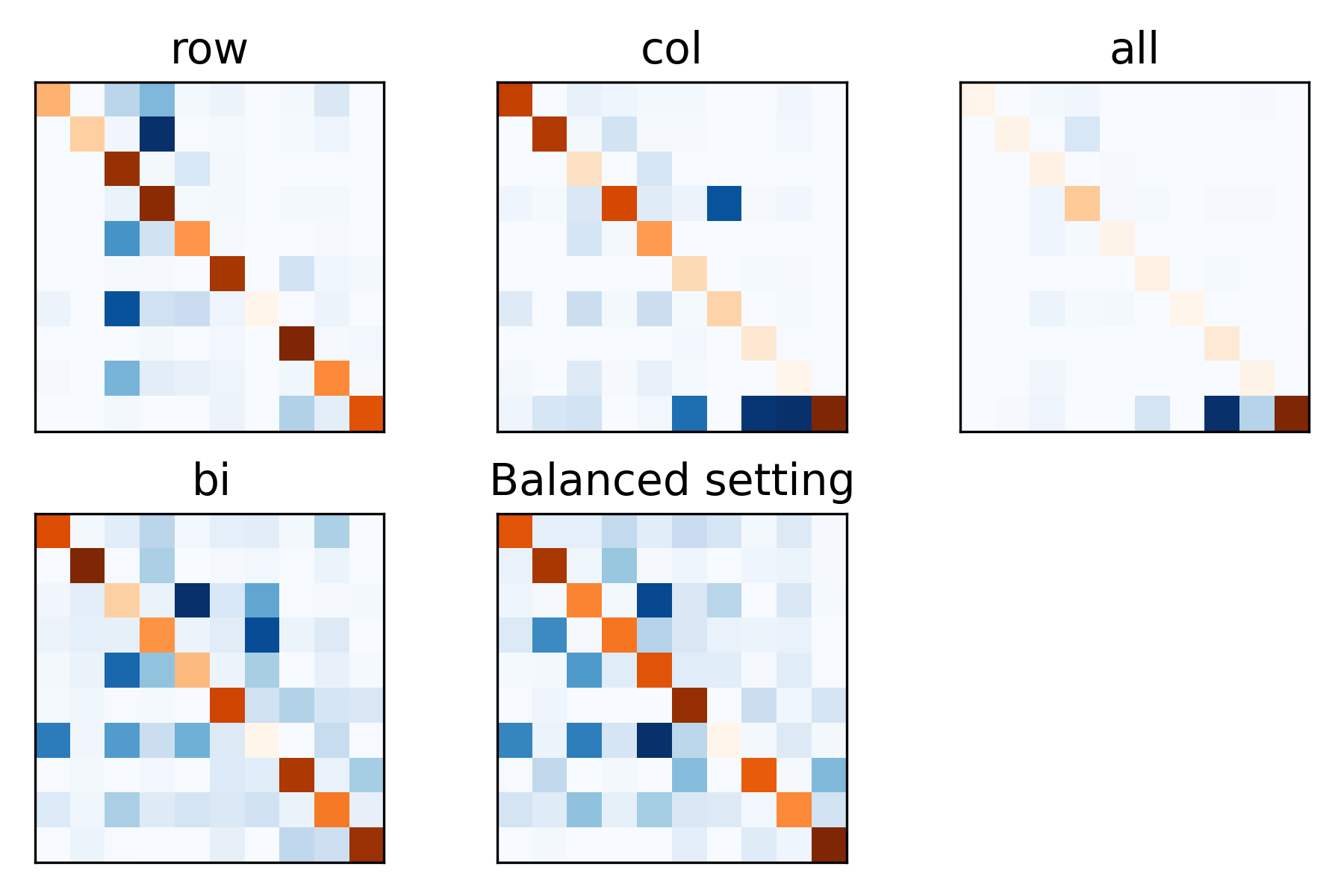}
    \caption{One instance (seed 0) on Fashion-MNIST with $\alpha = 0.3$. 
    Bi-normalization best recovers class relationships under distribution shifts. 
    Diagonal variations are shown in orange; off-diagonal in blue, with darker shades indicating higher values.}
    \label{exp1}
\end{figure}

\section{EMPIRICAL RESULTS}\label{Empirical Results}
Figure~\ref{plot} shows the results of Experiment~1 (Subfigure~\ref{fig:subfig1}) and Experiment~2 (Subfigures~\ref{fig:subfig2} and \ref{fig:subfig3}), described in Section~\ref{Baselines and Experiments}.

\textbf{Experiment 1.} This experiment evaluates how well different normalizations recover class similarities. Starting from a confusion matrix obtained under an imbalanced setting, normalization aims to approximate the one obtained under a balanced setting.

Across datasets and heterogeneity levels, Subfigure~\ref{fig:subfig1} shows that bi-normalization consistently achieves the highest overlap, outperforming other methods. As heterogeneity increases, (i) the gap between bi-normalization and the other methods tends to widen, and (ii) the ability to recover class similarity diminishes regardless of the normalization method.

This behavior is expected: the greater the imbalance between the training and test sets, the more distribution bias is introduced into the matrix. This bias tends to obscure class similarities, making $\row$, $\col$, and $\all$ normalizations unsuitable. Moreover, higher heterogeneity leads to prediction and test class distributions that deviate further from a balanced setting, making normalization procedures less reliable, as discussed in Section~\ref{IS}.

Figure~\ref{exp1} further illustrates that bi-normalization captures the class similarity patterns present in the balanced confusion matrix, whereas other methods often produce vertical or horizontal stripes. 

By definition, bi-normalization is expected to achieve the lowest KL divergence. Figure~\ref{fig:kl} in the Appendix confirms this, providing further evidence of its ability to recover class similarities.

\textbf{Experiment 2.} This experiment evaluates the correspondence between the organization of class representations in the model’s latent space and normalized confusion matrices.

Subfigure~\ref{fig:subfig2} demonstrates a clear correspondence between $\operatorname{GCM}\!{\omega_r}$ and row normalization. Overlap increases as heterogeneity grows. We observe similar results for $\operatorname{GCM}\!{\omega_a}$ with all normalizations, and for $\operatorname{GCM}\!{\omega_c}$ with column normalization (see Appendix).

Subfigure~\ref{fig:subfig3} shows that the correspondence between $\operatorname{GCM}\!{\omega_b}$ and bi-normalization is weaker. This makes sense, as the granularity differs from the other approaches. For instance, in $\operatorname{GCM}\!{\omega_r}$, all embedded points of the same label receive the same weight, while in $\operatorname{GCM}\!{\omega_b}$, embedded points are weighted by both label and prediction. This indicates that spatial separation is finer under $\operatorname{GCM}\!{\omega_b}$ weighting than with other weightings. Nonetheless, bi-normalization still exhibits higher correspondence than other methods.

This experiment confirms the validity of the geometric interpretation of normalizations described in Section \ref{geom}.

\section{CONCLUSION}\label{Conclusion and Future Work}
We revisited confusion matrix normalization and showed that bi-normalization provides a principled way to isolate class similarity from distribution bias. While the use of IPF for contingency tables is well established in other fields, its value for interpreting classifier behavior in machine learning has been largely overlooked. We demonstrated that bi-normalization reveals class similarities more clearly than standard methods, especially in heterogeneous settings.

We also showed that normalization corresponds to an importance sampling strategy and introduced a geometric interpretation of class representations in latent space, offering deeper insight into what normalization reveals about model behavior.

\subsection*{Acknowledgements}
We would like to thank Olivier Mbarek and Eric Lombardi, who are responsible for the PAGODA platform at LIRIS, for their valuable support and assistance. This work was supported by the French government managed by the Agence Nationale de la Recherche (ANR) through France 2030 program with the reference ANR-23-PEIA-005 (REDEEM project).

\bibliographystyle{plainnat}
\bibliography{References}

@inproceedings{STL10,
  title={An analysis of single-layer networks in unsupervised feature learning},
  author={Coates, Adam and Ng, Andrew and Lee, Honglak},
  booktitle={Proceedings of the fourteenth international conference on artificial intelligence and statistics},
  pages={215--223},
  year={2011},
  organization={JMLR Workshop and Conference Proceedings}
}

@article{krstinic2020multi,
  title={Multi-label classifier performance evaluation with confusion matrix},
  author={Krstini{\'c}, Damir and Braovi{\'c}, Maja and {\v{S}}eri{\'c}, Ljiljana and Bo{\v{z}}i{\'c}-{\v{S}}tuli{\'c}, Dunja},
  journal={Computer Science \& Information Technology},
  volume={1},
  year={2020}
}

@article{A,
  title={A dynamic model of classifier competence based on the local fuzzy confusion matrix and the random reference classifier},
  author={Trajdos, Pawel and Kurzynski, Marek},
  journal={International Journal of Applied Mathematics and Computer Science},
  volume={26},
  number={1},
  pages={175--189},
  year={2016}
}

@article{C,
  title={Bayes metaclassifier and Soft-confusion-matrix classifier in the task of multi-label classification},
  author={Trajdos, Pawel and Majak, Marcin},
  journal={arXiv preprint arXiv:1901.08827},
  year={2019}
}

@inproceedings{D,
  title={Randomized reference classifier with Gaussian distribution and soft confusion matrix applied to the improving weak classifiers},
  author={Trajdos, Pawel and Kurzynski, Marek},
  booktitle={Progress in Computer Recognition Systems 11},
  pages={326--336},
  year={2020},
  organization={Springer}
}

@inproceedings{E,
  title={Soft confusion matrix classifier for stream classification},
  author={Trajdos, Pawel and Kurzynski, Marek},
  booktitle={International conference on computational science},
  pages={3--17},
  year={2021},
  organization={Springer}
}

@article{krstinic2024multi,
  title={Multi-Label Confusion Tensor},
  author={Krstini{\'c}, Damir and Skelin, Ana Kuzmani{\'c} and Slapni{\v{c}}ar, Ivan and Braovi{\'c}, Maja},
  journal={IEEE access},
  year={2024},
  publisher={IEEE}
}

@book{sammut2011encyclopedia,
  title={Encyclopedia of machine learning},
  author={Sammut, Claude and Webb, Geoffrey I},
  year={2011},
  publisher={Springer Science \& Business Media}
}

@article{erbani2024confusion,
  title={Confusion matrices: A unified theory},
  author={Erbani, Johan and Portier, Pierre-{\'E}douard and Egyed-Zsigmond, El{\"o}d and Nurbakova, Diana},
  journal={IEEE Access},
  year={2024},
  publisher={IEEE}
}

@article{idel2016review,
  title={A review of matrix scaling and Sinkhorn's normal form for matrices and positive maps},
  author={Idel, Martin},
  journal={arXiv preprint arXiv:1609.06349},
  year={2016}
}

@inproceedings{gortler2022neo,
  title={Neo: Generalizing confusion matrix visualization to hierarchical and multi-output labels},
  author={G{\"o}rtler, Jochen and Hohman, Fred and Moritz, Dominik and Wongsuphasawat, Kanit and Ren, Donghao and Nair, Rahul and Kirchner, Marc and Patel, Kayur},
  booktitle={Proceedings of the 2022 CHI Conference on Human Factors in Computing Systems},
  pages={1--13},
  year={2022}
}

@article{ireland1968contingency,
  title={Contingency tables with given marginals},
  author={Ireland, C Terrance and Kullback, Solomon},
  journal={Biometrika},
  volume={55},
  number={1},
  pages={179--188},
  year={1968},
  publisher={Oxford University Press}
}

@article{hardin1997statistical,
  title={Statistical significance and normalized confusion matrices},
  author={Hardin, Perry J and Shumway, J Matthew},
  journal={Photogrammetric engineering and remote sensing},
  volume={63},
  number={6},
  pages={735--739},
  year={1997},
  publisher={[Falls Church, Va.] American Society of Photogrammetry.}
}

@article{congalton2001accuracy,
  title={Accuracy assessment and validation of remotely sensed and other spatial information},
  author={Congalton, Russell G},
  journal={International journal of wildland fire},
  volume={10},
  number={4},
  pages={321--328},
  year={2001},
  publisher={CSIRO Publishing}
}

@article{congalton1991review,
  title={A review of assessing the accuracy of classifications of remotely sensed data},
  author={Congalton, Russell G},
  journal={Remote sensing of environment},
  volume={37},
  number={1},
  pages={35--46},
  year={1991},
  publisher={Elsevier}
}

@inproceedings{kurras2015symmetric,
  title={Symmetric iterative proportional fitting},
  author={Kurras, Sven},
  booktitle={Artificial Intelligence and Statistics},
  pages={526--534},
  year={2015},
  organization={PMLR}
}

@inproceedings{nagineni2010line,
  title={On line client-wise cohort set selection for speaker verification using iterative normalization of confusion matrices},
  author={Nagineni, Srikanth and Hegde, Rajesh M},
  booktitle={2010 18th European Signal Processing Conference},
  pages={576--580},
  year={2010},
  organization={IEEE}
}

@article{deming1940least,
  title={On a least squares adjustment of a sampled frequency table when the expected marginal totals are known},
  author={Deming, W Edwards and Stephan, Frederick F},
  journal={The Annals of Mathematical Statistics},
  volume={11},
  number={4},
  pages={427--444},
  year={1940},
  publisher={JSTOR}
}

@inproceedings{thaper2002dynamic,
  title={Dynamic multidimensional histograms},
  author={Thaper, Nitin and Guha, Sudipto and Indyk, Piotr and Koudas, Nick},
  booktitle={Proceedings of the 2002 ACM SIGMOD international conference on Management of data},
  pages={428--439},
  year={2002}
}

@article{erbani2025weighted,
  title={Weighted Loss Methods for Robust Federated Learning under Data Heterogeneity},
  author={Erbani, Johan and Mokhtar, Sonia Ben and Portier, Pierre-Edouard and Egyed-Zsigmond, Elod and Nurbakova, Diana},
  journal={arXiv preprint arXiv:2506.09824},
  year={2025}
}

@article{hsu2019measuring,
  title={Measuring the effects of non-identical data distribution for federated visual classification},
  author={Hsu, Tzu-Ming Harry and Qi, Hang and Brown, Matthew},
  journal={arXiv preprint arXiv:1909.06335},
  year={2019}
}

@inproceedings{beery2020synthetic,
  title={Synthetic examples improve generalization for rare classes},
  author={Beery, Sara and Liu, Yang and Morris, Dan and Piavis, Jim and Kapoor, Ashish and Joshi, Neel and Meister, Markus and Perona, Pietro},
  booktitle={Proceedings of the ieee/cvf winter conference on applications of computer vision},
  pages={863--873},
  year={2020}
}

@article{zhu2022deep,
  title={Deep span representations for named entity recognition},
  author={Zhu, Enwei and Liu, Yiyang and Li, Jinpeng},
  journal={arXiv preprint arXiv:2210.04182},
  year={2022}
}

@article{krizhevsky2012imagenet,
  title={Imagenet classification with deep convolutional neural networks},
  author={Krizhevsky, Alex and Sutskever, Ilya and Hinton, Geoffrey E},
  journal={Advances in neural information processing systems},
  volume={25},
  year={2012}
}

@article{buda2018systematic,
  title={A systematic study of the class imbalance problem in convolutional neural networks},
  author={Buda, Mateusz and Maki, Atsuto and Mazurowski, Maciej A},
  journal={Neural networks},
  volume={106},
  pages={249--259},
  year={2018},
  publisher={Elsevier}
}

@article{leevy2018survey,
  title={A survey on addressing high-class imbalance in big data},
  author={Leevy, Joffrey L and Khoshgoftaar, Taghi M and Bauder, Richard A and Seliya, Naeem},
  journal={Journal of Big Data},
  volume={5},
  number={1},
  pages={1--30},
  year={2018},
  publisher={Springer}
}

@article{saye2015high,
  title={High-order quadrature methods for implicitly defined surfaces and volumes in hyperrectangles},
  author={Saye, Robert I},
  journal={SIAM Journal on Scientific Computing},
  volume={37},
  number={2},
  pages={A993--A1019},
  year={2015},
  publisher={SIAM}
}

@book{scott2015multivariate,
  title={Multivariate density estimation: theory, practice, and visualization},
  author={Scott, David W},
  year={2015},
  publisher={John Wiley \& Sons}
}

@book{boyd2004convex,
  title={Convex optimization},
  author={Boyd, Stephen P and Vandenberghe, Lieven},
  year={2004},
  publisher={Cambridge university press}
}

@inproceedings{allouah2023fixing,
  title={Fixing by mixing: A recipe for optimal byzantine ml under heterogeneity},
  author={Allouah, Youssef and Farhadkhani, Sadegh and Guerraoui, Rachid and Gupta, Nirupam and Pinot, Rafa{\"e}l and Stephan, John},
  booktitle={International Conference on Artificial Intelligence and Statistics},
  pages={1232--1300},
  year={2023},
  organization={PMLR}
}

@article{lecun1998gradient,
  title={Gradient-based learning applied to document recognition},
  author={LeCun, Yann and Bottou, L{\'e}on and Bengio, Yoshua and Haffner, Patrick},
  journal={Proceedings of the IEEE},
  volume={86},
  number={11},
  pages={2278--2324},
  year={1998},
  publisher={Ieee}
}

@article{xiao2017fashion,
  title={Fashion-mnist: a novel image dataset for benchmarking machine learning algorithms},
  author={Xiao, Han and Rasul, Kashif and Vollgraf, Roland},
  journal={arXiv preprint arXiv:1708.07747},
  year={2017}
}

@techreport{krizhevsky2009learning,
  title={Learning multiple layers of features from tiny images},
  author={Krizhevsky, Alex},
  institution={University of Toronto},
  year={2009}
}

@article{tokdar2010importance,
  title={Importance sampling: a review},
  author={Tokdar, Surya T and Kass, Robert E},
  journal={Wiley Interdisciplinary Reviews: Computational Statistics},
  volume={2},
  number={1},
  pages={54--60},
  year={2010},
  publisher={Wiley Online Library}
}

@article{fernando2021dynamically,
  title={Dynamically weighted balanced loss: class imbalanced learning and confidence calibration of deep neural networks},
  author={Fernando, K Ruwani M and Tsokos, Chris P},
  journal={IEEE Transactions on Neural Networks and Learning Systems},
  volume={33},
  number={7},
  pages={2940--2951},
  year={2021},
  publisher={IEEE}
}

@article{deepak2023brain,
  title={Brain tumor categorization from imbalanced MRI dataset using weighted loss and deep feature fusion},
  author={Deepak, S and Ameer, PM},
  journal={Neurocomputing},
  volume={520},
  pages={94--102},
  year={2023},
  publisher={Elsevier}
}

@article{chamseddine2022handling,
  title={Handling class imbalance in COVID-19 chest X-ray images classification: Using SMOTE and weighted loss},
  author={Chamseddine, Ekram and Mansouri, Nesrine and Soui, Makram and Abed, Mourad},
  journal={Applied Soft Computing},
  volume={129},
  pages={109588},
  year={2022},
  publisher={Elsevier}
}

@article{wu2022deep,
  title={Deep convolution neural network with weighted loss to detect rice seeds vigor based on hyperspectral imaging under the sample-imbalanced condition},
  author={Wu, Na and Weng, Shizhuang and Chen, Jinxin and Xiao, Qinlin and Zhang, Chu and He, Yong},
  journal={Computers and Electronics in Agriculture},
  volume={196},
  pages={106850},
  year={2022},
  publisher={Elsevier}
}

@article{ghosh2024class,
  title={The class imbalance problem in deep learning},
  author={Ghosh, Kushankur and Bellinger, Colin and Corizzo, Roberto and Branco, Paula and Krawczyk, Bartosz and Japkowicz, Nathalie},
  journal={Machine Learning},
  volume={113},
  number={7},
  pages={4845--4901},
  year={2024},
  publisher={Springer}
}

@article{tian2024novel,
  title={A novel data augmentation approach to fault diagnosis with class-imbalance problem},
  author={Tian, Jilun and Jiang, Yuchen and Zhang, Jiusi and Luo, Hao and Yin, Shen},
  journal={Reliability Engineering \& System Safety},
  volume={243},
  pages={109832},
  year={2024},
  publisher={Elsevier}
}

@article{temraz2022solving,
  title={Solving the class imbalance problem using a counterfactual method for data augmentation},
  author={Temraz, Mohammed and Keane, Mark T},
  journal={Machine Learning with Applications},
  volume={9},
  pages={100375},
  year={2022},
  publisher={Elsevier}
}

@inproceedings{aggarwal2021minority,
  title={Minority class oriented active learning for imbalanced datasets},
  author={Aggarwal, Umang and Popescu, Adrian and Hudelot, C{\'e}line},
  booktitle={2020 25th International Conference on Pattern Recognition (ICPR)},
  pages={9920--9927},
  year={2021},
  organization={IEEE}
}
\onecolumn
\section*{Checklist}

\begin{enumerate}
  \item For all models and algorithms presented, check if you include:
  \begin{enumerate}
    \item A clear description of the mathematical setting, assumptions, algorithm, and/or model. [\textbf{Yes}/No/Not Applicable] 
    \item An analysis of the properties and complexity (time, space, sample size) of any algorithm. [\textbf{Yes}/No/Not Applicable]  The only algorithm used is IPF, which, as established in the literature, converges linearly in our setting.
    \item (Optional) Anonymized source code, with specification of all dependencies, including external libraries. [\textbf{Yes}/No/Not Applicable]  Publicly available.
  \end{enumerate}

  \item For any theoretical claim, check if you include:
  \begin{enumerate}
    \item Statements of the full set of assumptions of all theoretical results. [\textbf{Yes}/No/Not Applicable]
    \item Complete proofs of all theoretical results. [\textbf{Yes}/No/Not Applicable] All proofs are provided in Appendix.
    \item Clear explanations of any assumptions. [\textbf{Yes}/No/Not Applicable]    
  \end{enumerate}

  \item For all figures and tables that present empirical results, check if you include:
  \begin{enumerate}
    \item The code, data, and instructions needed to reproduce the main experimental results (either in the supplemental material or as a URL). [\textbf{Yes}/No/Not Applicable]
    \item All the training details (e.g., data splits, hyperparameters, how they were chosen). [\textbf{Yes}/No/Not Applicable]
    \item A clear definition of the specific measure or statistics and error bars (e.g., with respect to the random seed after running experiments multiple times). [\textbf{Yes}/No/Not Applicable]
    \item A description of the computing infrastructure used. (e.g., type of GPUs, internal cluster, or cloud provider). [Yes/No/\textbf{Not Applicable}]
  \end{enumerate}

  \item If you are using existing assets (e.g., code, data, models) or curating/releasing new assets, check if you include:
  \begin{enumerate}
    \item Citations of the creator If your work uses existing assets. [\textbf{Yes}/No/Not Applicable]
    \item The license information of the assets, if applicable. [Yes/No/\textbf{Not Applicable}]
    \item New assets either in the supplemental material or as a URL, if applicable. [Yes/No/\textbf{Not Applicable}]
    \item Information about consent from data providers/curators. [Yes/No/\textbf{Not Applicable}]
    \item Discussion of sensible content if applicable, e.g., personally identifiable information or offensive content. [Yes/No/\textbf{Not Applicable}]
  \end{enumerate}

  \item If you used crowdsourcing or conducted research with human subjects, check if you include:
  \begin{enumerate}
    \item The full text of instructions given to participants and screenshots. [Yes/No/\textbf{Not Applicable}]
    \item Descriptions of potential participant risks, with links to Institutional Review Board (IRB) approvals if applicable. [Yes/No/\textbf{Not Applicable}]
    \item The estimated hourly wage paid to participants and the total amount spent on participant compensation. [Yes/No/\textbf{Not Applicable}]
  \end{enumerate}

\end{enumerate}

\newpage
\appendix
The supplementary materials are organized as follows:  
\textbf{Section~1} provides additional results that further support our propositions.
\textbf{Section~2} describes the Iterative Proportional Fitting (IPF) algorithm and an equivalent formulation known as the RAS algorithm, whose variations with respect to IPF are later exploited.  
\textbf{Section~3} details the overlap metric used in our experiments.  
\textbf{Section~4} presents experimental details, including the multidimensional histogram setup and the model architectures.  
\textbf{Section~5} provides proofs of key properties such as information preservation, idempotence, and class distribution invariance.  
\textbf{Section~6} describes additional weighting schemes mentioned in Subsection~\nameref{IS}.
\textbf{Section~7} derives and proves the explicit formula for the Geometric Confusion Matrix (GCM).  
\textbf{Section~8} demonstrates an auxiliary lemma used in the preceding sections.

\section{ADDITIONAL RESULTS}
Figure~\ref{plot_other} presents the remaining results of Experiment~2 (see Subsection~\nameref{Baselines and Experiments}) with target matrices $\operatorname{GCM}_a$ and $\operatorname{GCM}_c$. We observe trends consistent with those obtained for $\operatorname{GCM}_r$: $\operatorname{GCM}_a$ aligns closely with the $\all$ normalization (Subfigure~\ref{plot_other:subfig1}), while $\operatorname{GCM}_c$ aligns closely with the $\col$ normalization (Subfigure~\ref{plot_other:subfig2}). These results further support the geometric interpretation of normalizations.

Figure~\ref{fig:kl} reports the results for Experiment~1 (see Subsection~\nameref{Baselines and Experiments}) using the KL divergence instead of the overlap metric. It shows that bi-normalization still achieves the highest similarity with the confusion matrix under a balanced setting, as expected.

\begin{figure}[h!]
    \centering
    % First subfigure
    \begin{subfigure}[b]{0.49\linewidth}
        \centering
        \includegraphics[width=0.99\linewidth]{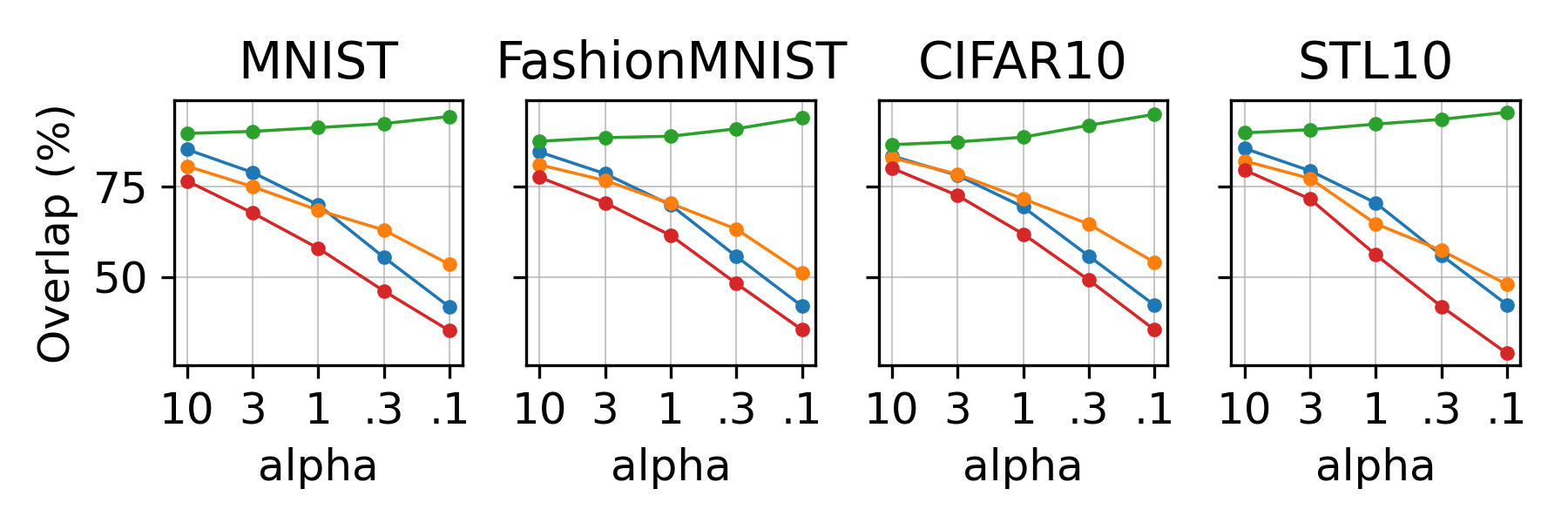}
        \vspace{-.6cm}
        \caption{Target: $\operatorname{GCM}\!{\omega_a}$}
        \label{plot_other:subfig1}
    \end{subfigure}
    \begin{subfigure}[b]{0.49\linewidth}
        \centering
        \includegraphics[width=0.99\linewidth]{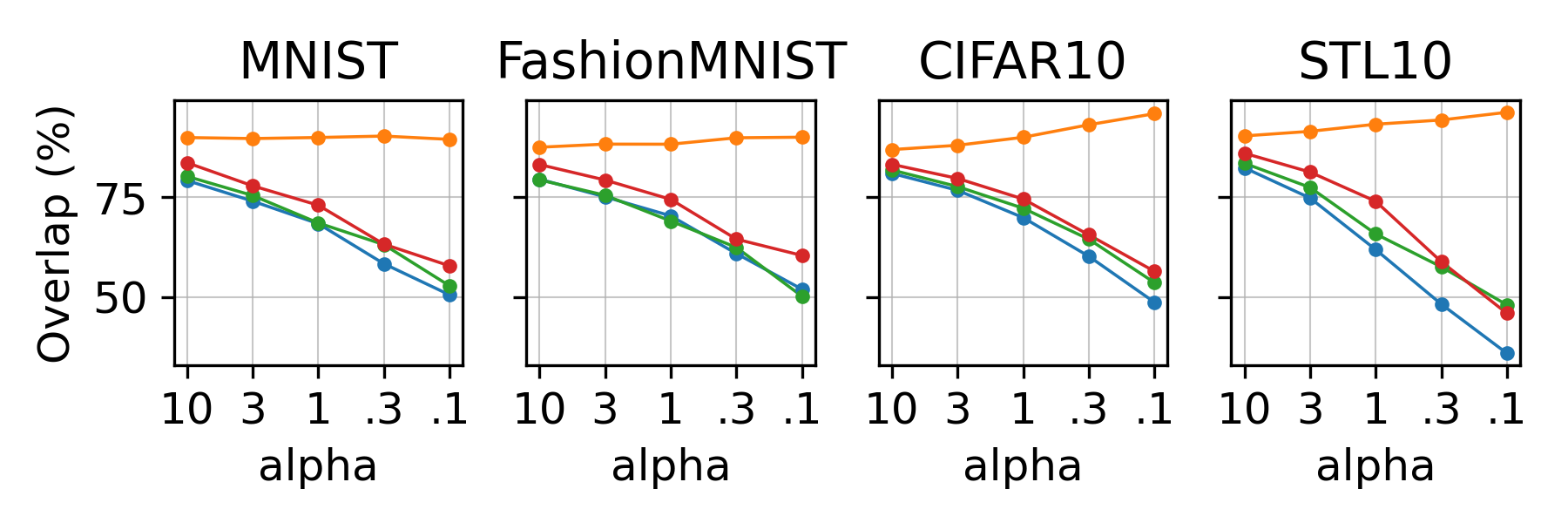}
        \vspace{-.6cm}
        \caption{Target: $\operatorname{GCM}\!{\omega_c}$}
        \label{plot_other:subfig2}
    \end{subfigure}
    \caption{Overlap (\%) between target matrices and normalized matrices. Legend:~\textcolor{color4}{$\;\bullet\!$}~bi 
    \textcolor{color1}{$\;\bullet\!$}~row 
    \textcolor{color2}{$\;\bullet\!$}~col 
    \textcolor{color3}{$\;\bullet\!$}~all
    }
    \label{plot_other}
\end{figure}
\begin{figure}
    \centering
    % First subfigure
    \centering
    \includegraphics[width=0.49\linewidth]{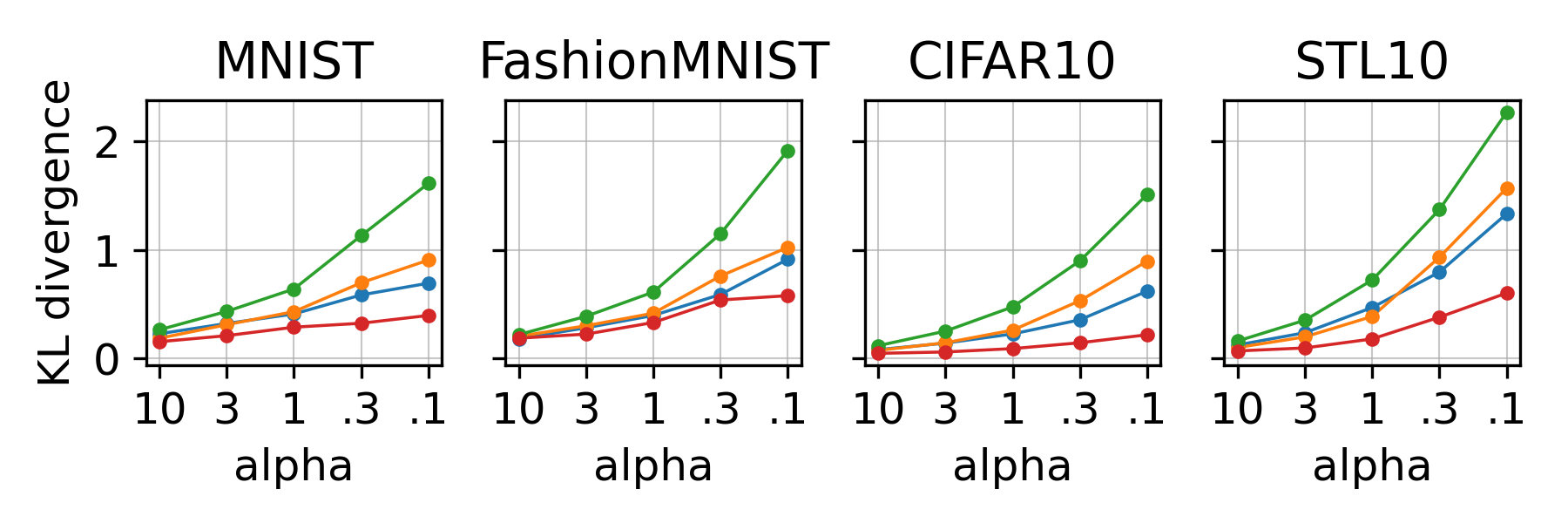}
    \vspace{-.2cm}
    \caption{KL divergence (\%) between matrix from balanced setting and normalized matrices. Legend:~\textcolor{color4}{$\;\bullet\!$}~bi 
    \textcolor{color1}{$\;\bullet\!$}~row 
    \textcolor{color2}{$\;\bullet\!$}~col 
    \textcolor{color3}{$\;\bullet\!$}~all}
    \label{fig:kl}
\end{figure}

\section{IPF \& RAS ALGORITHMS}\label{al}
This section presents the IPF (Algorithm~\ref{algorithm}) and RAS (Algorithm~\ref{algorithm2}) procedures. 

According to Lemma~\ref{lem} (Section~\ref{equi}), there exist positive diagonal matrices ${D^l}$ and ${D^r}$ such that $\bis(M) = {D^l} M {D^r}$. The RAS method estimates ${D^l}$ and ${D^r}$. This algorithm will be used in the following section and is equivalent to Algorithm~\ref{algorithm}; see~\citep{idel2016review} for further details.

\begin{algorithm}[t!]
\caption{Iterative Proportional Fitting (IPF)}
\label{algorithm}
\begin{algorithmic}[1]
\REQUIRE Initial matrix $M \in \mathbb{R}_{>0}^{C \times C}$, target row sums $u \in \mathbb{R}_{>0}^{C}$, target column sums $v \in \mathbb{R}_{>0}^{C}$ such that $u_+ = v_+$, tolerance $\epsilon > 0$, maximum number of steps $T$
\STATE Initialize $Q^{(0)} \gets M$, $t \gets 0$
\REPEAT
    \FOR{each row $i = 1$ to $C$}
        \STATE $Q^{(t+1)}_{i,:} \gets Q^{(t)}_{i,:} \cdot \frac{u_i}{Q^{(t)}_{i+}}$
    \ENDFOR
    \FOR{each column $j = 1$ to $C$}
        \STATE $Q^{(t+2)}_{:,j} \gets Q^{(t+1)}_{:,j} \cdot \frac{v_j}{Q^{(t+1)}_{+j}}$
    \ENDFOR
    \STATE $t \gets t + 2$
\UNTIL{$\|Q^{(t)}_{:+} - u\|_1 + \|Q^{(t)}_{+:} - v\|_1 \leq \epsilon$ \OR $t \geq T$}
\RETURN $\widehat{Q} \gets Q^{(t)}$
\end{algorithmic}
\end{algorithm}

\begin{algorithm}[t!]
\caption{RAS method}
\label{algorithm2}
\begin{algorithmic}
\REQUIRE Positive matrix $M \in \mathbb{R}_{>0}^{C \times C}$, target row sums $r \in \mathbb{R}_{>0}^C$, target column sums $c \in \mathbb{R}_{>0}^C$ such that $r_+ = c_+$, tolerance $\epsilon > 0$, maximum number of steps $T$
\STATE Initialize ${D^r}^{(0)} \gets I$, $t \gets 0$
\REPEAT
    \FOR{each $i = 1$ to $C$}
        \STATE $D_{ii}^{l\, (t+1)} \gets \dfrac{r_i}{\sum_j M_{ij} D_{jj}^{r\, (t)}}$
    \ENDFOR
    \FOR{each $j = 1$ to $C$}
        \STATE $D_{jj}^{r\, (t+2)} \gets \dfrac{c_j}{\sum_i M_{ij} D_{ii}^{l\, (t+1)}}$
    \ENDFOR
    \STATE $t \gets t + 2$
\UNTIL{$\|({D}^{l\, (t)} M {D}^{r\, (t)})_{:+} - r\|_1 + \|({D}^{l\, (t)} M {D}^{r\, (t)})_{+:} - c\|_1 \leq \epsilon$ \OR $t \geq T$}
\RETURN $\widehat{D}^l \gets {D}^{l\, (t)}$, $\widehat{D}^r \gets {D}^{r\, (t)}$
\end{algorithmic}
\end{algorithm}

\section{OVERLAP METRIC}
This section explains the rationale for normalizing confusion matrices when computing the Overlap metric, and shows its equivalence with the $L_1$ distance.

\subsection{Why Normalize?}
The entries of a confusion matrix are only meaningful relative to each other; absolute values lack interpretability. For example, $P_{ii} = 3$ indicates perfect classification if $P_{i+} = 3$, but poor performance if $P_{i+} = 30$. More generally, for any $\lambda > 0$, $P$ and $\lambda P$ represent the same model behavior.

Comparisons between two matrices are meaningful when they share the same total sum, i.e., $P_{++} = Q_{++}$. For instance, $P_{ii} = 1$ may seem worse than $Q_{ii} = 10$, but if $P_{++} = 3$ and $Q_{++} = 30$, the relative accuracies are identical: $P_{ii}/P_{++} = Q_{ii}/Q_{++}$.

For consistency, we compare normalized confusion matrices: $\all(P)=P / P_{++}$ and $\all(Q)=Q / Q_{++}$. This ensures that behaviors are preserved and enables consistent comparisons.

\subsection{Equivalence with $L_1$ Distance}
Let $P$ and $Q$ be two matrices such that $P_{++} = Q_{++} = 1$. Then,
$$
\begin{aligned}
\|P-Q\|_1=\sum_{ij}|P_{ij}-Q_{ij}|&=\sum_{ij}\max(P_{ij}, Q_{ij})-\min(P_{ij}, Q_{ij})\\
&=\sum_{ij:P_{ij}\geq Q_{ij}}P_{ij}-\min(P_{ij}, Q_{ij})\quad\quad+\sum_{ij:Q_{ij}> P_{ij}}Q_{ij}-\min(P_{ij}, Q_{ij})\\
&=\sum_{ij}P_{ij}-\min(P_{ij}, Q_{ij})+\sum_{ij}Q_{ij}-\min(P_{ij}, Q_{ij})\\
&=2-2\operatorname{Overlap}(P,Q)
\end{aligned}
$$
Thus, the Overlap metric is directly related to the $L_1$ distance.

\section{EXPERIMENTS DETAILS}
This section presents details on the multidimensional histogram setup and the model architectures.

\subsection{HISTOGRAM IN EXPERIMENTS}
First, the projection space is divided into a regular grid of $m$-dimensional hyperrectangles,\footnote{In geometry, a hyperrectangle generalizes a rectangle (in 2D) and a rectangular cuboid (in 3D) to higher dimensions~\citep{saye2015high}.} which define the histogram bins. We then construct non-overlapping histograms~\citep{thaper2002dynamic} associated with the clusters $\mathcal{C}_k$ and $\widehat{\mathcal{C}}_k$ for $k=1,\ldots, C$.

Considering a cluster $\mathcal{C}$ and its unweighted histogram, each bin height represents the number of points from $\mathcal{C}$ that fall within the bin (i.e., hyperrectangle) under consideration. In the weighted case, the bin height instead corresponds to the sum of the point weights, using one of these weighting schemes $\omega_a$, $\omega_r$, $\omega_c$, or $\omega_b$.

In the experiments, we fix the projected space dimensionality to $m = 10$ and specify the bin width accordingly. Following the recommendation of \citep{scott2015multivariate} (see Subsection 3.4, Eq. (3.66)), we adopt the optimal bin width
$3.5\ \sigma_k\ n^{-1/(2 + m)},$
where $\sigma_k$ denotes the empirical standard deviation along the $k$-th dimension, $n$ is the sample size, and $m$ is the data dimensionality.

\subsection{MODEL ARCHITECTURES}\label{Model Architectures}
We use the following abbreviations to describe our architectures: $L(n)$ denotes a fully connected linear layer with $n$ outputs; $R$ is a ReLU activation; $C(c)$ represents a 2D convolutional layer with $c$ output channels, kernel size $5$ for CIFAR datasets and $3$ otherwise, padding $0$ for CIFAR and $1$ otherwise, and stride $1$; $M$ denotes 2D max pooling with kernel size $2$; $B$ stands for batch normalization; and $D$ represents dropout with a fixed probability of $0.25$.

In line with \citet{allouah2023fixing}, the model architectures are defined as follows. For MNIST and Fashion-MNIST, the architecture is $(1, 28 \times 28)$ – $C(8)$ – $R$ – $M$ – $C(16)$ – $R$ – $M$ – $L(64)$ – $R$ – $L(10)$. For CIFAR-10 and STL-10, the architecture is $(3, 32 \times 32)$ – $C(64)$ – $R$ – $B$ – $C(64)$ – $R$ – $B$ – $M$ – $D$ – $C(128)$ – $R$ – $B$ – $C(128)$ – $R$ – $B$ – $M$ – $D$ – $L(128)$ – $R$ – $D$ – $L(10)$.

\section{INFORMATION PRESERVATION, IDEMPOTENCE \& CLASS DISTRIBUTION INVARIANCE}
This section demonstrates that normalization methods satisfy the properties of information preservation, idempotence, and class distribution invariance.

\subsection{Information Preservation for Standards Normalizations}\label{equality}
We state that, for a positive confusion matrix $M$, the normalizations minimize the KL divergence under specific constraints:
\begin{equation*}
\all(M)\in
\myargmin{
P \in \mathbb{R}_{>0}^{C \times C}: \\
P_{i+} =\frac{M_{i+}}{M_{++}}, P_{+j} = \frac{M_{+j}}{M_{++}} \ \forall\, i,j
}
D_{\mathrm{KL}}(P \| M),\quad
\row(M)\in
\myargmin{
P \in \mathbb{R}_{>0}^{C \times C}: \\
P_{i+} =1, P_{+j} = \sum_{i}\frac{M_{ij}}{M_{i+}} \ \forall\, i,j
}
D_{\mathrm{KL}}(P \| M),\quad 
\col(M)\in
\myargmin{
P \in \mathbb{R}_{>0}^{C \times C}: \\
P_{i+} =\sum_{j}\frac{M_{ij}}{M_{+j}}, P_{+j} = 1\ \forall\, i,j
}
D_{\mathrm{KL}}(P \| M).
\end{equation*}

We observe that these normalizations can be expressed as the product of diagonal matrices and the original confusion matrix:
\begin{equation}\label{diagmat}
\row(M)\!=\!\operatorname{diag}(\tfrac{1}{M_{1+}}, \ldots, \tfrac{1}{M_{C+}}) M I,\quad 
\col(M)\!=\!I M \operatorname{diag}(\tfrac{1}{M_{+1}}, \ldots, \tfrac{1}{M_{+C}}),\text{ and }\all(M)\!=\!I M I / M_{++}.
\end{equation}
According to Lemma~\ref{lem} (Section~\ref{equi}), these normalization procedures indeed minimize the KL divergence, which concludes the proof.

\subsection{Idempotence \& Class Distribution Invariance of Bi-Normalization}
In Proposition 1, we state that the bi-normalization satisfies idempotence,
$$
\bis \circ \bis(M) = \bis(M),
$$
and class distribution invariance, that is, for any diagonal matrices $A, B \in \mathbb{R}_{>0}^{C \times C}$,
$$
\bis(A M B) = \bis(M).
$$
We prove these properties in what follows.

\textbf{Idempotence.} By definition, $\bis \circ \bis(M)$ is the solution to the problem
\begin{equation*}
\myargmin{
P \in \mathbb{R}_{>0}^{C \times C}: \\
P_{i+} = P_{+j} = 1 \ \forall\, i,j
}
D_{\mathrm{KL}}(P \| \bis(M)) 
\end{equation*}
Moreover, we observe that $I \bis(M) I$ satisfies the marginal constraints. By Lemma~\ref{lem}, $I \bis(M) I$ solves the above problem, and since the solution is unique, we have
$$
\bis \circ \bis(M) = I \,\bis(M)\, I = \bis(M),
$$
which establishes the idempotence property.

\textbf{Class Distribution Invariance.} By Lemma~\ref{lem}, there exist diagonal matrices ${D^l}$ and ${D^r}$ such that
$$
\operatorname{bi}(A M B) = {D^l} A M B {D^r}.
$$
By definition of the bi-normalization, the row and column sums satisfy
$$
({D^l} A M B {D^r})_{i+} = ({D^l} A M B {D^r})_{+j} = 1 \quad \forall i,j.
$$
Additionally, ${D^l} A$ and $B {D^r}$ are diagonal.
Therefore, by Lemma~\ref{lem}, $\operatorname{bi}(A M B)$ is the unique solution to
\begin{equation*}
\myargmin{
P \in \mathbb{R}_{>0}^{C \times C}: \\
P_{i+} =1,  P_{+j} = 1 \ \forall\, i,j
}
D_{\mathrm{KL}}(P \| M)
\end{equation*}
This implies
$
\operatorname{bi}(A M B) = \operatorname{bi}(M),
$
completing the proof.

\section{IMPORTANCE SAMPLING}
This section details the weighting schemes mentioned in Subsection~\nameref{IS}.

From the equalities presented in Equation~(\ref{diagmat}) (Subsection~\ref{equality}) and by applying the same procedure described in the main part of the paper, it directly follows that
$$
\omega_a : (y,\hat{y}) \mapsto \frac{1}{M_{++}}\quad\ \text{and}\quad \all(M) = \sum_{k=1}^N \omega_a(y_k,\hat{y}_k) E_{y_k \hat{y}_k},
$$
and
$$
\omega_c : (y,\hat{y}) \mapsto \frac{1}{M_{+\hat{y}}}\quad\ \text{and}\quad  \col(M) = \sum_{k=1}^N \omega_c(y_k,\hat{y}_k) E_{y_k \hat{y}_k}.
$$

According to Lemma~\ref{lem} (Section~\ref{equi}), there exist positive diagonal matrices ${D^l}$ and ${D^r}$ such that $\bis(M) = {D^l} M {D^r}$, using Algorithm~\ref{algorithm2} (Section~\ref{al}), we estimate these matrices by matrices $\widehat{D}^l$ and ${\widehat{D}^r}$. It follows
$$
\omega_b : (y,\hat{y}) \mapsto {\widehat{D}^l_{y}\widehat{D}^r_{\hat{y}}}\quad\ \text{and}\quad \bis(M) = \sum_{k=1}^N \omega_b(y_k,\hat{y}_k) E_{y_k \hat{y}_k}.
$$
\section{GEOMETRIC CONFUSION MATRIX}
This section establishes the volumes of the various scaled histograms introduced previously, and compares the analytical form of the classical confusion matrix with that of the geometric confusion matrix.

\subsection{Normalized Histograms Under Weighting Schemes}
In Section~\nameref{geom}, we state that, under specific weighting schemes, some histograms are normalized up to a scaling factor. We prove this statement below.

We first introduce the following lemma:
\begin{lemma}\label{lem2}
Let $\mathcal{C}$ be a cluster of points in the model’s latent space, based on dataset $\mathcal{D}$. Let $P$ be a projection matrix, and let $\omega$ be the weighting function. Then,
\[
\lambda(V(\mathcal{C}, \omega)) = r \sum_{(x, y) \in \mathcal{D}} \omega(y,\hat{y})\mathds{1}_{P\varepsilon(x) \in \mathcal{C}},
\]
where $\lambda$ denotes the Lebesgue measure, $\mathds{1}$ is the indicator function, and $r$ the volume of hyperrectangles which split the projection space.
\end{lemma}
\begin{proof}
We begin by decomposing the histogram into a union of hyperrectangles which grid the projection space, denoted by $\mathcal{V}$. Since $\mathcal{D}$ contains a finite number of points, the covering set $\mathcal{V}$ is also finite. By construction of the weighted histogram, we have
\[
\begin{aligned}
V(\mathcal{C}, \omega) &= \bigcup_{v \in \mathcal{V}} \left(V(\mathcal{C}, \omega) \cap v\right) \\
&= \bigcup_{v \in \mathcal{V}} 
\underset{\text{bin/hyperrectangle splitting the projection space}}{\underbrace{r^v_1 \times \cdots \times r^v_m}}
\times 
\underset{\text{bin height}}{\underbrace{\left[0, \sum_{(x, y) \in \mathcal{D}} \omega (y, \hat{y})\mathds{1}_{P\varepsilon(x) \in \mathcal{C} \cap v} \right]}}
\end{aligned}
\]
where $r^v_1, \ldots, r^v_m$ denote the intervals describing the sides of each hyperrectangle.

Its Lebesgue measure equals the sum of the measures of the individual hyperrectangles:
\[
\lambda(V(\mathcal{C}, \omega)) =\sum_{v \in \mathcal{V}} \lambda(r^v_1 \times \cdots \times r^v_m \times \left[0, \sum_{(x, y) \in \mathcal{D}} \omega(y, \hat{y})\mathds{1}_{P\varepsilon(x) \in \mathcal{C} \cap v} \, \right])
\]
The Lebesgue measure of a hyperrectangle is the product of its side lengths. Only the heights vary, while the other side lengths are bin-independent. Let $r = \prod_{i=1}^m \lambda(r_i^v)$. Then,
\[
\lambda(V(\mathcal{C}, \omega)) 
= \sum_{v \in \mathcal{V}} r \sum_{(x, y) \in \mathcal{D}} \omega(y, \hat{y})
\mathds{1}_{P\varepsilon(x) \in \mathcal{C} \cap v} = r \sum_{(x, y) \in \mathcal{D}} \omega(y, \hat{y})
\mathds{1}_{P\varepsilon(x) \in \mathcal{C}}
\]
which concludes the proof.
\end{proof}

According to Lemma~\ref{lem2}, we have
$$
\begin{aligned}
&\lambda(V(\mathcal{C}_i, \omega_r)) = r \sum_{(x, y) \in \mathcal{D}} \omega_r(y,\hat{y})\mathds{1}_{P\varepsilon(x) \in \mathcal{C}_i}=r \sum_{(x, y) \in \mathcal{D}}\frac{1}{M_{i+}}\mathds{1}_{P\varepsilon(x) \in \mathcal{C}_i}=r,
\\
&\lambda(V(\widehat{\mathcal{C}}_j, \omega_c)) = r \sum_{(x, y) \in \mathcal{D}} \omega_c(y,\hat{y})\mathds{1}_{P\varepsilon(x) \in \widehat{\mathcal{C}}_j}=r \sum_{(x, y) \in \mathcal{D}}\frac{1}{M_{+j}}\mathds{1}_{P\varepsilon(x) \in \widehat{\mathcal{C}}_j}=r,
\\
&\lambda(V(\mathcal{C}_i, \omega_b)) = r \sum_{(x, y) \in \mathcal{D}} \omega_b(y,\hat{y})\mathds{1}_{P\varepsilon(x) \in \mathcal{C}_i}=r \sum_{(x, y) \in \mathcal{D}}\widehat{D}^l_i\widehat{D}^r_{\hat{y}}\mathds{1}_{P\varepsilon(x) \in \mathcal{C}_i}=r\bis(M)_{i+}=r, \text{ and}
\\
&\lambda(V(\widehat{\mathcal{C}}_j, \omega_b)) = r \sum_{(x, y) \in \mathcal{D}} \omega_b(y,\hat{y})\mathds{1}_{P\varepsilon(x) \in \widehat{\mathcal{C}}_j}=r \sum_{(x, y) \in \mathcal{D}}\widehat{D}^l_y\widehat{D}^r_{j}\mathds{1}_{P\varepsilon(x) \in \widehat{\mathcal{C}}_j}=r\bis(M)_{+j}=r.
\end{aligned}
$$
In that sense, these histograms are normalized using weighting schemes up to a scaling factor $r$, corresponding to the fixed volume of the hyperrectangles partitioning the projection space.

\subsection{Geometric Confusion Matrix}
The Geometric Confusion Matrix is defined as
\begin{equation*}
\begin{aligned}
\bigg(\operatorname{GCM}_{\omega}\bigg)_{ij} &= \lambda\bigg(V(\mathcal{C}_i,\omega)\cap V(\widehat{\mathcal{C}}_j, \omega)\bigg)
\end{aligned}
\end{equation*}
We begin by deriving an explicit formula for $\operatorname{GCM}_{\omega}$. Following the same approach as in proof of Lemma~\ref{lem2}, and noting that the grid of hyperrectangles partitioning the projection space is shared across all histograms, the intersection of two histograms is given by:
$$
\begin{aligned}
V(\mathcal{C}_i, \omega)\cap V(\widehat{\mathcal{C}}_j, \omega)&=
\bigcup_{v \in \mathcal{V}} r^v_1 \times \cdots \times r^v_m \times \left[0, \sum_{(x, y) \in \mathcal{D}} \omega(y, \hat{y})\mathds{1}_{P\varepsilon(x) \in \mathcal{C}_i \cap v}  \right]
\quad\bigcap\\
&\quad\quad\quad\quad\quad\quad\bigcup_{v \in \mathcal{V}} r^v_1 \times \cdots \times r^v_m \times \left[0, \sum_{(x, y) \in \mathcal{D}} \omega(y, \hat{y})\mathds{1}_{P\varepsilon(x) \in \widehat{\mathcal{C}}_j \cap v} \right]\\
&=\bigcup_{v \in \mathcal{V}} r^v_1 \times \cdots \times r^v_m \times \Bigg[0,\min\Bigg(\sum_{(x, y) \in \mathcal{D}} \omega(y, \hat{y})\mathds{1}_{P\varepsilon(x) \in \mathcal{C}_i \cap v},\sum_{(x, y) \in \mathcal{D}} \omega(y, \hat{y})\mathds{1}_{P\varepsilon(x) \in \widehat{\mathcal{C}}_j \cap v} \Bigg)\Bigg]
\end{aligned}
$$
As a result, the $(i,j)$ entry of the geometric confusion matrix becomes:
\begin{equation*}
\begin{aligned}
\bigg(\operatorname{GCM}_{l,p}\bigg)_{ij} = r\ \sum_{v\in\mathcal{V}}\min\Bigg(\sum_{(x, y) \in \mathcal{D}} \omega(y, \hat{y})\mathds{1}_{P\varepsilon(x) \in \mathcal{C}_i \cap v}, \sum_{(x, y) \in \mathcal{D}} \omega(y, \hat{y})\mathds{1}_{P\varepsilon(x) \in \widehat{\mathcal{C}}_j \cap v}\Bigg)
\end{aligned}
\end{equation*}

The normalized confusion matrix can also be expressed as a sum over grid cells, analogously to $\operatorname{GCM}_{\omega}$. Let $\operatorname{norm}$ denote a normalization operator among $\all$, $\row$, $\col$, and $\bis$, and let $\omega$ be its associated weighting scheme. It is straightforward to obtain:
\begin{equation*}
\operatorname{norm}(M)_{ij} = \sum_{(x, y) \in \mathcal{D}} \omega(i,j)\mathds{1}_{y=i}\mathds{1}_{\hat{y}=j} 
= \sum_{(x, y) \in \mathcal{D}} \omega(i,j)\mathds{1}_{P\varepsilon(x) \in \mathcal{C}_i \cap \widehat{\mathcal{C}}_j}
= \sum_{v\in\mathcal{V}}\sum_{(x, y) \in \mathcal{D}}\omega(i,j)\mathds{1}_{P\varepsilon(x) \in \mathcal{C}_i \cap \widehat{\mathcal{C}}_j \cap v}
\end{equation*}

This explicit formula for $\operatorname{GCM}_{\omega}$ highlights how the geometric confusion matrix incorporates spatial information. The geometric confusion matrix compares in each grid cell the number of points labeled $i$ (regardless of prediction) with the number of points predicted as $j$ (regardless of true label). As a result, in a grid cell containing only points with label $i$ and prediction $j$, the contribution to $\operatorname{GCM}_{\omega}$ matches that of the standard confusion matrix. In contrast, in a cell containing a mix of labels and predictions, the minimum operation produces a contribution that differs from the simple count of points labeled $i$ and predicted as $j$—that is, from the contribution in the standard confusion matrix.

\section{EQUIVALENCE SCALING LEMMA}\label{equi}
This section demonstrates the following lemma:
\begin{lemma}\label{lem}
Let M be a positive matrix in $\mathbb{R}_{>0}^{C \times C}$, and $r$ and $c$, be two positive vectors in $\mathbb{R}_{>0}^{C}$ such that $r_{+}=c_{+}$. Then, any matrix $M^*$ of the form $M^*={D^l} M {D^r}$, where ${D^l}$ and ${D^r}$ are positive diagonal matrices such that
$$
({D^l} M {D^r})_{i+} = r_i \quad \text{and} \quad ({D^l} M {D^r})_{+j} = c_j \quad \text{for all } i,j,
$$
is the unique solution to the following problem:
\begin{equation*}
M^*\in 
\myargmin{
P \in \mathbb{R}_{>0}^{C \times C}: \\
P_{i+} = r_i,  P_{+j} = c_j \ \forall\, i,j
}
D_{\mathrm{KL}}(P \| M)
\end{equation*}
\end{lemma}
\vspace{-0.5cm}
Note that this result is already known in a different form, referred to as equivalence scaling (see Section 3. Different approaches to equivalence scaling in~\citep{idel2016review}). We do not claim to introduce a new result with this lemma; rather, we were unable to find this particular formulation in the literature.

\subsection{Problem Reformulation}
We use the method of Lagrange multipliers. We define the row constraints for $i=1,\ldots,C$ as:
$$
\begin{array}{llll}
h_i: & \mathbb{R}_{>0}^{C \times C} & \to & \mathbb{R} \\
     & P & \mapsto & P_{i+} - r_i
\end{array}
$$
and the column constraints for $j=1,\ldots,C$ as:
$$
\begin{array}{llll}
g_j: & \mathbb{R}_{>0}^{C \times C} & \to & \mathbb{R} \\
     & P & \mapsto & P_{+j} - c_j
\end{array}
$$

Since adding a constant to the objective function does not affect the arguments of the minima, we define $f(P) := D_{\mathrm{KL}}(P \| M) - 1$ to simplify the derivative calculations. 

The optimization problem becomes:
\begin{equation*}
M^*\in 
\myargmin{
P \in \mathbb{R}_{>0}^{C \times C}: \\
h_i(P) = g_j(P) = 0 \ \forall\, i,j
}
f(P)
\end{equation*}
The functions $f$, $h_i$, and $g_j$ are continuously differentiable.

\subsection{Karush-Kuhn-Tucker Conditions}
To find a solution, we apply the Karush-Kuhn-Tucker (KKT) conditions (see \textit{Section 5.5.3: KKT Optimality Conditions} in \citep{boyd2004convex}). These conditions are necessary for a point $M^*$ to be a local minimum when a constraint qualification is satisfied. In our case, the Linearity Constraint Qualification (LCQ) holds, since the constraints $h_i$ and $g_j$ are affine.

Moreover, since $f$ is convex (as shown in a subsection below), and the constraint set is convex, the KKT conditions are also sufficient for optimality~\citep{boyd2004convex}. Thus, any point satisfying the KKT conditions is a global minimum.

In other words, any feasible matrix $P^*$ and multipliers $\mu^*$ and $\nu^*$ satisfying:
$$
\nabla_P L(P^*, \mu^*, \nu^*) = 0
$$
is a global minimum, where the Lagrangian is given by:
$$
L(P, \mu, \nu) = f(P) + \sum_{i=1}^C \mu_i h_i(P) + \sum_{j=1}^C \nu_j g_j(P)
$$
Since $f$ is strictly convex (Subsection \textit{Convexity of $f$}), this minimum is also unique.

\subsection{Solution}
Let $P$ be an admissible point. We compute the gradient of the Lagrangian\footnote{Though the gradient is a vector, we index it here by $(i,j)$ for clarity.}:
\begin{equation*}
\nabla_P L(P, \mu, \nu)_{ij} = \left[\ln\left(\frac{P_{ij}}{M_{ij}}\right) + \mu_i + \nu_j\right]_{ij}
\end{equation*}
Setting the gradient to zero gives:
$$
\begin{aligned}
&\quad\ln\left(\frac{P_{ij}}{M_{ij}}\right) + \mu_i + \nu_j = 0\quad
\Leftrightarrow \quad
P_{ij} = \exp(-\mu_i) M_{ij} \exp(-\nu_j)
\end{aligned}
$$
We define diagonal matrices ${D^l}$ and ${D^r}$ with entries:
$$
({D^l})_{ii} = \exp(-\mu_i), \qquad ({D^r})_{jj} = \exp(-\nu_j)
$$
Then the solution takes the form:
$$
P = {D^l} M {D^r}
$$

In conclusion, any matrix of the form ${D^l} M {D^r}$ that satisfies the given row and column constraints is the unique solution to the minimization problem.

\subsection{Convexity of $f$}
Let $P, P' \in \mathbb{R}_{>0}^{C \times C}$ be two distinct matrices, and let $\gamma \in (0,1)$. The function $f$ is strictly convex if and only if
$$
f(\gamma P + (1 - \gamma) P') < \gamma f(P) + (1 - \gamma) f(P').
$$

We now prove this inequality. We first show that the KL divergence is strictly convex in $P$ when $M$ is fixed.
$$
\begin{aligned} 
D_{\mathrm{KL}}(\gamma P+ (1-\gamma) P^\prime \| M)&=\sum_{ij} \bigg(\gamma P+ (1-\gamma) P^\prime\bigg)\ln\bigg(\frac{\gamma P+ (1-\gamma) P^\prime}{\gamma M+ (1-\gamma) M}\bigg)\\
&<\sum_{ij} \gamma  P\ln\bigg(\frac{P}{M}\bigg) +  (1-\gamma)  P^\prime\ln\bigg(\frac{P^\prime}{M}\bigg)\\
&=\gamma D_{\mathrm{KL}}(P\| M)+ (1-\gamma)D_{\mathrm{KL}}(P^\prime\| M)
\end{aligned}
$$
where the strict inequality follows from the log-sum inequality, which becomes strict when $P \neq P'$. This implies that $f$ is strictly convex, which completes the proof.

\end{document}